\documentclass{article}

\usepackage{microtype}
\usepackage{graphicx}
\usepackage{subcaption}
\usepackage{booktabs}
\usepackage{hyperref}
\usepackage{algorithm}
\usepackage{algorithmic}
\usepackage{amssymb}
\usepackage{pifont}

\newcommand{\RETURN}{\STATE \textbf{return} }

\usepackage[accepted]{icml2026}

\newcommand{\cmark}{\ding{51}}
\newcommand{\xmark}{\ding{55}}
\usepackage{amsmath}
\usepackage{mathtools}
\usepackage{amsthm}
\usepackage{bm}
\usepackage{multicol}
\usepackage{multirow}
\usepackage[capitalize,noabbrev]{cleveref}

\theoremstyle{plain}
\newtheorem{theorem}{Theorem}[section]
\newtheorem{proposition}[theorem]{Proposition}
\newtheorem{lemma}[theorem]{Lemma}

\theoremstyle{definition}
\newtheorem{definition}[theorem]{Definition}
\newtheorem{assumption}[theorem]{Assumption}
\theoremstyle{remark}
\newtheorem{remark}[theorem]{Remark}

\icmltitlerunning{Safe Urban Traffic Control via Uncertainty-Aware Conformal Prediction and World-Model RL}

\begin{document}

\twocolumn[
\icmltitle{Safe Urban Traffic Control via Uncertainty-Aware Conformal Prediction and \\ World-Model Reinforcement Learning}

\icmlsetsymbol{equal}{*}

\begin{icmlauthorlist}
\icmlauthor{Joydeep Chandra}{tsinghua_bnrist}
\icmlauthor{Satyam Kumar Navneet}{independent}
\icmlauthor{Aleksandr Algazinov}{tsinghua_cst}
\icmlauthor{Yong Zhang}{tsinghua_bnrist}
\end{icmlauthorlist}

\icmlaffiliation{tsinghua_bnrist}{Dept. of CST, BNRIST, Tsinghua University, Beijing, China}
\icmlaffiliation{tsinghua_cst}{Dept. of CST, Tsinghua University, Beijing, China}
\icmlaffiliation{independent}{Independent Researcher, Bihar, India}

\icmlcorrespondingauthor{Joydeep Chandra}{joydeepc2002@gmail.com}
\icmlcorrespondingauthor{Satyam Kumar Navneet}{navneetsatyamkumar@gmail.com}

\icmlkeywords{Conformal Prediction, Safe Reinforcement Learning, Spatio-Temporal Forecasting}

\vskip 0.3in
]

\printAffiliationsAndNotice{}

\begin{abstract}
Urban traffic management demands systems that simultaneously predict future conditions, detect anomalies, and take safe corrective actions—all while providing reliability guarantees. We present STREAM-RL, a unified framework that introduces three novel algorithmic contributions: (1)  \textbf{PU-GAT$^+$}, an Uncertainty-Guided Adaptive Conformal Forecaster that uses prediction uncertainty to dynamically reweight graph attention via confidence-monotonic attention, achieving distribution-free coverage guarantees; (2)\textbf{ CRFN-BY}, a Conformal Residual Flow Network that models uncertainty-normalized residuals via normalizing flows with Benjamini-Yekutieli FDR control under arbitrary dependence; and (3)\textbf{ LyCon-WRL$^+$}, an Uncertainty-Guided Safe World-Model RL agent with Lyapunov stability certificates, certified Lipschitz bounds, and uncertainty-propagated imagination rollouts. To our knowledge, this is the first framework to propagate calibrated uncertainty from forecasting through anomaly detection to safe policy learning with end-to-end theoretical guarantees. Experiments on multiple real-world traffic trajectory data demonstrate that STREAM-RL achieves 91.4\% coverage efficiency, controls FDR at 4.1\% under verified dependence, and improves safety rate to 95.2\% compared to 69\% for standard PPO while achieving higher reward, with 23ms end-to-end inference latency.

\textit{\textcolor{red}{Preliminary work. Under review.}}
\end{abstract}
\begin{figure*}[t]
    \centering
    \includegraphics[width=1\linewidth]{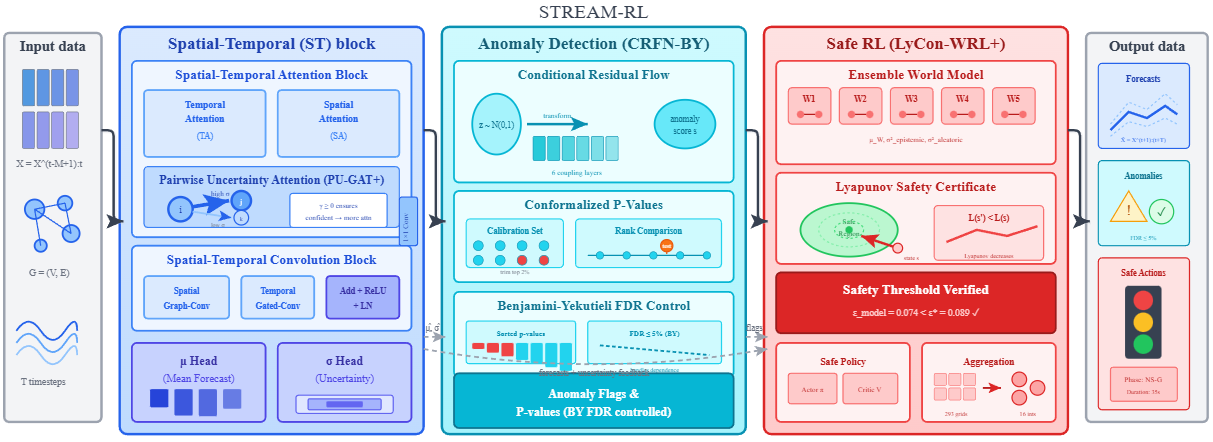}
    \caption{Architecture of STREAM-RL. The framework integrates: (1) PU-GAT+ for uncertainty-aware forecasting, (2) CRFN-BY for dependence-robust anomaly detection, and (3) LyCon-WRL+ for certified safe RL. Dashed lines show cross-module uncertainty propagation.}
    \label{fig:arch}    
\end{figure*}
\section{Introduction}
\label{sec:intro}

Urban traffic congestion costs the U.S. economy \$87 billion annually \citep{schrank2021urban}, while contributing 29\% of transportation-related emissions. Effective management requires: (i) \emph{forecasting} with calibrated uncertainty \citep{jiang2023graph,li2018dcrnn,guo2019attention}, (ii) \emph{anomaly detection} with controlled false discovery rates \citep{liu2022anomaly,pang2021deep}, and (iii) \emph{safe adaptive control} satisfying hard safety constraints \citep{wei2021recent,wei2019presslight,zheng2019learning}.

Existing integrated approaches suffer from three critical limitations:

\textbf{Limitation 1: Pseudo-uncertainty guidance in graph attention.} Prior uncertainty-aware GNNs \citep{wu2020connecting,bai2020adaptive,shang2021discrete,guo2021learning} employ temperature scaling $\alpha_{ij} \propto \exp(e_{ij}/(1 + \beta\sigma_i))$, which merely adjusts softmax sharpness---it cannot preferentially weight confident neighbors over uncertain ones.

\textbf{Limitation 2: Invalid FDR control under dependence.} The Benjamini-Hochberg procedure \citep{benjamini1995controlling} requires independence or PRDS among p-values. Spatial autocorrelation and temporal persistence in traffic data violate these assumptions \citep{fithian2022conditional,lei2018adapt}, rendering FDR guarantees invalid.

\textbf{Limitation 3: Unverified safe RL assumptions.} Lyapunov-based safe RL methods \citep{chow2018lyapunov,chow2019lyapunov,berkenkamp2017safe,fisac2019bridging} derive guarantees conditional on assumptions (bounded model error, Lipschitz Lyapunov function) that are rarely verified empirically.

We propose \textbf{STREAM-RL} with three methodologically complete contributions that address these limitations:

\textbf{Contribution 1: Constrained Pairwise Uncertainty-Guided Attention (PU-GAT$^+$).} We introduce attention coefficients that explicitly compare source and target uncertainties with \emph{provably monotonic} behavior:
\begin{equation}
\alpha_{ij} = \text{softmax}_j\Big(e_{ij} + \underbrace{\text{Softplus}(\tilde{\gamma})}_{\gamma \geq 0}(\sigma_i - \sigma_j) + \delta \cdot \mathbf{1}[j = i]\Big)
\label{eq:pugat_intro}
\end{equation}
The Softplus constraint ensures $\gamma \geq 0$, guaranteeing that when $\sigma_i > \sigma_j$ (source uncertain, neighbor confident), attention to neighbor $j$ strictly increases. We train via heteroscedastic Gaussian NLL with explicit calibration diagnostics (PIT histograms, reliability diagrams) before conformal wrapping.

\textbf{Contribution 2: Conformalized P-Value Construction (CRFN-BY).} We employ the Benjamini-Yekutieli procedure \citep{benjamini2001control} with \emph{precisely specified} p-value construction from flow-based anomaly scores:
\begin{equation}
p_t^{(i)} = \frac{1 + \sum_{j \in \mathcal{C}_{\text{trim}}} \mathbf{1}[s_j \geq s_t^{(i)}]}{1 + |\mathcal{C}_{\text{trim}}|}
\label{eq:pvalue_intro}
\end{equation}
where $s_t^{(i)} = -\log p_\theta(z_t^{(i)} | c_t^{(i)})$ is the context-conditioned negative log-likelihood and $\mathcal{C}_{\text{trim}}$ is the contamination-trimmed calibration set. This conformalized construction yields valid p-values under exchangeability, which BY then corrects for arbitrary dependence.

\textbf{Contribution 3: Complete Lyapunov Certificate Derivation (LyCon-WRL$^+$).} We employ spectral normalization \citep{miyato2018spectral} to obtain enforced upper bounds $\bar{L}_L$ and $\bar{J}_W$ on Lipschitz constants:
\begin{equation}
\varepsilon^* = \frac{\delta_{\text{slack}} + \kappa \cdot \bar{d}_C}{\bar{L}_L \cdot (1 + \bar{J}_W)}
\label{eq:eps_star_intro}
\end{equation}
where bounds are enforced via spectral normalization during training (not post-hoc estimates), addressing the gap between theoretical assumptions and practical verification raised in prior work \citep{gouk2021regularisation,fazlyab2019efficient}.

\section{Related Work}
\label{sec:related}

\textbf{Uncertainty-Aware Graph Neural Networks.} Bayesian GNNs \citep{zhang2019bayesian,hasanzadeh2020bayesian,pal2020nonparametric} propagate uncertainty but incur overhead. Heteroscedastic approaches \citep{wu2020connecting,gal2016dropout,lakshminarayanan2017simple} output uncertainty estimates for loss weighting, not attention modulation. Recent work on confidence-aware message passing \citep{feng2021kdd,zhao2020uncertainty,stadler2021graph} addresses robustness but not monotonic uncertainty guidance. Our PU-GAT$^+$ explicitly compares source-target uncertainties with constrained monotonicity.

\textbf{Conformal Prediction for Spatio-Temporal Data.} STACI \citep{xu2024staci} introduces topology-aware calibration; ACI \citep{gibbs2021adaptive,zaffran2022adaptive} addresses distribution shift. Recent methods incorporate travel-time adjacency \citep{jensen2024travel}, weather conditioning \citep{liu2024weather}, and heterogeneous calibration \citep{romano2019conformalized,sesia2021conformal,feldman2021improving}. These focus on calibration design; our contribution is uncertainty-guided \emph{representation learning} before conformal wrapping.

\textbf{FDR Control Under Dependence.} Beyond BY \citep{benjamini2001control}, knockoff filters \citep{barber2015controlling,candes2018panning}, local FDR \citep{efron2012large,storey2003positive}, and adaptive procedures \citep{blanchard2020post,ramdas2019unified,lei2018adapt} offer alternatives. Graph-aware FDR procedures \citep{lihua2017general,ramdas2017online} and e-values \citep{vovk2021evalues,grunwald2024safe} improve power under known dependence structure. We adopt BY for simplicity and validity under \emph{arbitrary} dependence, acknowledging its conservatism as a power-validity trade-off (see Section~\ref{sec:discussion}).

\textbf{Safe Model-Based RL.} LAMBDA \citep{as2022constrained}, LBPO \citep{sikchi2022lyapunov}, and CPO \citep{achiam2017constrained} combine constraints with learned models. Conservative critics \citep{bharadhwaj2021conservative}, risk-sensitive objectives \citep{ma2021conservative,tang2019worst}, and reachability methods \citep{fisac2019bridging,bajcsy2019efficient} provide alternatives. Our contribution is \emph{complete derivation} of model-error thresholds from certified Lipschitz constants via spectral normalization \citep{miyato2018spectral,gouk2021regularisation}.

\textbf{Traffic Forecasting and Control.} Deep learning approaches \citep{li2018dcrnn,yu2018spatio,wu2019graph,song2020spatial,zheng2020gman,chen2021z,jiang2021dltraff,cirstea2022towards} achieve strong performance. RL for traffic signal control \citep{wei2019presslight,wei2018intellilight,chen2020toward,zhang2022expression,oroojlooy2020deep} shows promise but lacks safety guarantees. Decision-focused learning \citep{wang2024aster,wilder2019melding,elmachtoub2022smart} jointly optimizes prediction and decision, while we propagate uncertainty through sequential modules.

\section{Problem Formulation}
\label{sec:problem}

\textbf{Traffic Network Representation.} We model urban traffic as a directed graph $\mathcal{G} = (\mathcal{V}, \mathcal{E}, \mathbf{A})$ with $N = |\mathcal{V}|$ spatial regions (grid cells or road segments). At each time step $t$, we observe feature matrix $\mathbf{X}_t \in \mathbb{R}^{N \times F}$ containing traffic flow, speed, and density measurements. The adjacency matrix $\mathbf{A} \in \mathbb{R}^{N \times N}$ encodes spatial connectivity.

\textbf{Research Questions.} We address three interconnected challenges:

\textbf{RQ1 (Forecasting):} Given historical observations $\{\mathbf{X}_{t-\tau+1}, \ldots, \mathbf{X}_t\}$, predict future states $\mathbf{Y}_{t+1:t+H}$ with prediction intervals $[\mathbf{L}, \mathbf{U}]$ satisfying:
\begin{equation}
\mathbb{P}(Y_{t+h}^{(i)} \in [L_{t+h}^{(i)}, U_{t+h}^{(i)}]) \geq 1 - \alpha, \quad \forall i, h
\label{eq:coverage}
\end{equation}
with intervals as tight as possible (high coverage efficiency).

\textbf{RQ2 (Anomaly Detection):} Produce p-values $\{p_t^{(i)}\}_{i=1}^N$ for each spatial location such that the false discovery rate is controlled:
\begin{equation}
\text{FDR} = \mathbb{E}\left[\frac{|\{i: H_i^0 \text{ true}, p_t^{(i)} \leq \tau\}|}{|\{i: p_t^{(i)} \leq \tau\}| \vee 1}\right] \leq \alpha
\label{eq:fdr}
\end{equation}
under arbitrary spatial-temporal dependence among p-values.

\textbf{RQ3 (Safe Control):} Learn policy $\pi_\theta: \mathcal{S} \to \mathcal{A}$ maximizing expected return subject to safety constraints:
\begin{align}
\max_\theta \; & \mathbb{E}_{\pi_\theta}\left[\sum_{t=0}^\infty \gamma^t r_t\right] \label{eq:rl_obj}\\
\text{s.t.} \; & C_k(\pi_\theta) \leq d_k, \quad k = 1, \ldots, K \label{eq:rl_const}
\end{align}
with \emph{verified} Lyapunov decrease certificates ensuring constraint satisfaction.

\section{Methodology}
\label{sec:method}

\subsection{PU-GAT$^+$: Confidence-Monotonic Attention}
\label{sec:pugat}

\subsubsection{Motivation: Why Temperature Scaling Fails}

Consider the temperature-scaled attention used in prior work:
\begin{equation}
\alpha_{ij}^{\text{temp}} = \frac{\exp(e_{ij} / (1 + \beta\sigma_i))}{\sum_{k \in \mathcal{N}(i)} \exp(e_{ik} / (1 + \beta\sigma_i))}
\label{eq:temp_scaling}
\end{equation}
The uncertainty $\sigma_i$ affects only the \emph{temperature} of the softmax distribution---it cannot distinguish between confident and uncertain \emph{neighbors}.

\begin{proposition}[Temperature Scaling Limitation]
\label{prop:temp_fail}
For temperature-scaled attention (Eq.~\ref{eq:temp_scaling}), the relative attention ratio between any two neighbors $j, k \in \mathcal{N}(i)$ is independent of their uncertainties:
$\alpha_{ij}^{\text{temp}}/\alpha_{ik}^{\text{temp}} = \exp(e_{ij})/\exp(e_{ik})$.
\end{proposition}

\begin{proof}
The temperature $(1 + \beta\sigma_i)$ cancels in the ratio since it affects all logits equally.
\end{proof}

\subsubsection{Pairwise Uncertainty Mechanism with Monotonicity Constraint}

\begin{definition}[PU-GAT$^+$ Attention]
\label{def:pugat}
Let $\mathbf{h}_i, \mathbf{h}_j \in \mathbb{R}^d$ be node embeddings and $\sigma_i, \sigma_j \in \mathbb{R}^+$ be uncertainty estimates. The pairwise uncertainty-guided attention is:
\begin{align}
e_{ij} &= \text{LeakyReLU}(\mathbf{a}^\top [\mathbf{W}\mathbf{h}_i \| \mathbf{W}\mathbf{h}_j]) \label{eq:base_attn}\\
u_{ij} &= \gamma \cdot (\sigma_i - \sigma_j) + \delta \cdot \mathbf{1}[j = i] \label{eq:pairwise_term}\\
\alpha_{ij} &= \frac{\exp(e_{ij} + u_{ij})}{\sum_{k \in \mathcal{N}(i)} \exp(e_{ik} + u_{ik})} \label{eq:final_attn}
\end{align}
where $\gamma = \text{Softplus}(\tilde{\gamma})$ with learnable $\tilde{\gamma} \in \mathbb{R}$, ensuring $\gamma \geq 0$, and $\delta \in \mathbb{R}$ is a learnable \emph{self-loop bias} that modulates self-attention strength.
\end{definition}

\textbf{Role of $\delta$ (Self-Loop Bias).} The parameter $\delta$ controls how much a node attends to itself relative to neighbors. When $\delta > 0$, self-attention is encouraged; when $\delta < 0$, neighbor information is preferred. Learned jointly with $\gamma$.

\begin{proposition}[Confidence-Monotonic Attention]
\label{prop:confidence_monotonic}
For PU-GAT$^+$ with $\gamma \geq 0$, the attention ratio satisfies:
\begin{equation}
\frac{\alpha_{ij}}{\alpha_{ik}} = \exp\big((e_{ij} - e_{ik}) + \gamma(\sigma_k - \sigma_j)\big)
\label{eq:ratio}
\end{equation}
This ratio is: (1) increasing in $\sigma_k$, (2) decreasing in $\sigma_j$, and (3) independent of $\sigma_i$.
\end{proposition}

The proof is in Appendix~\ref{app:proof_monotonic}.

\begin{remark}[Interpretation]
Proposition~\ref{prop:confidence_monotonic} establishes that PU-GAT$^+$ implements \emph{confidence-monotonic attention}: the relative attention to neighbor $j$ over $k$ increases when $j$ is more confident ($\sigma_j$ smaller) or $k$ is less confident ($\sigma_k$ larger). Crucially, this property is independent of the source node's uncertainty $\sigma_i$---the mechanism compares \emph{neighbor} confidences, not source confidence.
\end{remark}

\subsubsection{Uncertainty Propagation Across Layers}

Each PU-GAT$^+$ layer updates both features and uncertainties:
\begin{align}
\mathbf{h}_i^{(\ell)} &= \sigma_{\text{act}}\left(\sum_{j \in \mathcal{N}(i)} \alpha_{ij}^{(\ell)} \mathbf{W}^{(\ell)} \mathbf{h}_j^{(\ell-1)}\right) \label{eq:feat_update}\\
\sigma_i^{(\ell)} &= \text{Softplus}\left(\mathbf{w}_\sigma^{(\ell)\top} \mathbf{h}_i^{(\ell)} + b_\sigma^{(\ell)}\right) \label{eq:unc_update}
\end{align}
The Softplus activation ensures $\sigma_i^{(\ell)} > 0$. Uncertainties from layer $\ell - 1$ influence attention in layer $\ell$, creating an uncertainty-aware message passing scheme.

\subsubsection{Dual-Stream Temporal Decomposition}

Traffic exhibits distinct trend (daily/weekly patterns) and residual (events, incidents) components. We decompose inputs:
\begin{align}
\mathbf{X}_{\text{trend},t} &= \sum_{k=-w}^{w} \alpha_k \mathbf{X}_{t+k}, \quad \alpha_k = \text{Softmax}(\mathbf{w}_{\text{MA}})_k \label{eq:trend}\\
\mathbf{X}_{\text{res},t} &= \mathbf{X}_t - \mathbf{X}_{\text{trend},t} \label{eq:residual}
\end{align}
Each stream has a dedicated PU-GAT$^+$ encoder. Combined uncertainty accounts for correlation:
\begin{equation}
\sigma_{\text{total}}^2 = \sigma_{\text{trend}}^2 + \sigma_{\text{res}}^2 + 2\rho \cdot \sigma_{\text{trend}} \cdot \sigma_{\text{res}}
\label{eq:combined_unc}
\end{equation}
with learnable correlation $\rho = \tanh(w_\rho) \in (-1, 1)$.

\textbf{Training.} We use heteroscedastic Gaussian NLL (Eq.~\ref{eq:het_nll} in Appendix~\ref{app:training}).

\subsubsection{Spatially-Adaptive Conformal Calibration}

Raw uncertainties $\hat{\sigma}$ from neural networks are typically miscalibrated. We apply split conformal prediction \citep{vovk2005algorithmic,romano2019conformalized} with spatial clustering for heterogeneous calibration:

\textbf{Step 1: Spatial Clustering.} Partition nodes into $K$ clusters $\{\mathcal{C}_k\}_{k=1}^K$ based on historical error distributions using k-means on error statistics $(\bar{e}_i, \text{std}(e_i), \text{skew}(e_i))$.

\textbf{Step 2: Per-Cluster Quantiles.} For each cluster $k$, compute conformity scores on calibration data: $R_t^{(i)} = |y_t^{(i)} - \hat{\mu}_t^{(i)}|/\hat{\sigma}_t^{(i)}$ for $i \in \mathcal{C}_k$, and quantile $q_k = \text{Quantile}_{(1-\alpha)(1 + 1/n_k)}(\{R_t^{(i)}\})$.

\textbf{Step 3: Adaptive Updates.} Following ACI \citep{gibbs2021adaptive}, we adapt $\alpha_t^{(k)}$ to distribution shift:
\begin{equation}
\alpha_{t+1}^{(k)} = \alpha_t^{(k)} + \gamma_{\text{ACI}} \cdot (\alpha - \text{err}_t^{(k)})
\label{eq:aci}
\end{equation}

\textbf{Exchangeability Justification.} Calibration and test windows are non-overlapping and separated by a gap of $\tau_{\text{gap}} = 24$ time steps (6 hours) to mitigate temporal dependence. Under $\beta$-mixing with $\beta(\tau) \leq C\rho^\tau$:

\begin{theorem}[Coverage Guarantee under Mixing]
\label{thm:coverage}
For cluster calibration sets of size $n_k \geq n_{\min}$, calibration-test gap $\tau_{\text{gap}}$, and ACI step size $\gamma_{\text{ACI}} \in (0,1)$:
\begin{equation}
\limsup_{T \to \infty} \frac{1}{T}\sum_{t=1}^T \mathbf{1}[Y_t \notin \mathcal{I}_t] \leq \alpha + O\left(\frac{C\rho^{\tau_{\text{gap}}}}{1-\rho}\right) + O(\gamma_{\text{ACI}})
\end{equation}
\end{theorem}

\emph{Proof.} See Appendix~\ref{app:proof_coverage}. The $\rho^{\tau_{\text{gap}}}$ term arises from the calibration-test temporal gap, making the TV bound well-controlled for $\tau_{\text{gap}} = 24$ steps. \hfill $\square$

\subsection{CRFN-BY: Dependence-Robust Anomaly Detection}
\label{sec:crfn}

\subsubsection{Uncertainty-Normalized Residuals}

We normalize forecast residuals by calibrated uncertainty:
\begin{equation}
z_t^{(i)} = \frac{y_t^{(i)} - \hat{\mu}_t^{(i)}}{\hat{\sigma}_t^{(i)} + \epsilon}
\label{eq:normalized_residual}
\end{equation}
with $\epsilon = 10^{-6}$ for numerical stability. Under correct calibration and Gaussian errors, $z \sim \mathcal{N}(0, 1)$ for normal (non-anomalous) data.

\subsubsection{Context-Conditioned Normalizing Flow}

We model the density of normalized residuals via a conditional Real-NVP flow \citep{dinh2017density,kobyzev2021normalizing,papamakarios2021normalizing}:
\begin{equation}
p_\theta(z | c) = p_0(f_\theta(z; c)) \cdot \left|\det \frac{\partial f_\theta}{\partial z}\right|
\label{eq:flow_density}
\end{equation}
where $p_0 = \mathcal{N}(0, I)$ is the base distribution and $f_\theta$ is a composition of affine coupling layers conditioned on context $c$. The context encodes spatial and temporal information:
\begin{equation}
c_t^{(i)} = \text{CrossAttn}(z_t^{(i)}, \{z_t^{(j)}\}_{j \in \mathcal{N}(i)}) \oplus \text{CausalTF}(\{z_s^{(i)}\}_{s=t-\tau}^{t-1})
\label{eq:context}
\end{equation}
Anomaly scores are negative log-likelihoods: $s_t^{(i)} = -\log p_\theta(z_t^{(i)} | c_t^{(i)})$.

\subsubsection{Conformalized P-Value Construction}

The mapping from anomaly scores to valid p-values is essential for FDR guarantees:

\begin{definition}[Conformalized P-Values]
\label{def:pvalue}
Let $\mathcal{C} = \{s_1, \ldots, s_n\}$ be calibration scores from verified normal data. For test score $s_t^{(i)}$:
\begin{equation}
p_t^{(i)} = \frac{1 + \sum_{j=1}^n \mathbf{1}[s_j \geq s_t^{(i)}]}{1 + n}
\label{eq:pvalue}
\end{equation}
\end{definition}

Under exchangeability, $\mathbb{P}(p_t^{(i)} \leq \alpha | H_0^{(i)}) \leq \alpha$ (Proposition~\ref{prop:pvalue_valid} in Appendix).

\subsubsection{Contamination-Robust Calibration}

The calibration set may contain undetected anomalies. We use trimmed calibration:

\begin{lemma}[Trimming Preserves Super-Uniformity]
\label{lem:trimming}
Let $\mathcal{C}_{\text{trim}} = \mathcal{C} \setminus \{s : s \geq q_{1-\tau}(\mathcal{C})\}$ and $p_{\text{trim}}$ be the corresponding p-value. Under the null with at most $\tau'$ contamination fraction in calibration:
\begin{equation}
\mathbb{P}_{P_0}(p_{\text{trim}}(s) \leq \alpha) \leq \alpha + \tau + \tau'
\end{equation}
\end{lemma}

\emph{Proof sketch.} See Appendix~\ref{app:proof_trimming}. Trimming removes potential contamination but shifts rank distribution by at most $\tau n$ positions. \hfill $\square$

\textbf{Selection of $\tau$ and $\tau'$.} We set $\tau = 0.02$ (2\%) based on domain knowledge that anomaly prevalence in urban traffic rarely exceeds 2\% per time step during normal operations \citep{zheng2014urban}. The bound $\tau' \leq \tau$ holds when contamination does not exceed the trim fraction---a mild condition satisfied when calibration data is collected from verified normal periods. This ensures that trimming removes potential outliers while preserving p-value validity with error at most $\alpha + 0.04$.

\subsubsection{Benjamini-Yekutieli FDR Control}

BH requires independence or PRDS among p-values. Traffic data exhibits spatial autocorrelation ($\text{Moran's } I > 0$) and temporal persistence ($\text{ACF}(1) > 0.7$), violating these assumptions. We employ BY \citep{benjamini2001control}:
\begin{equation}
\begin{aligned}
k^* &= \max\left\{k : p_{(k)} \leq \frac{k \cdot \alpha}{m \cdot c_m}\right\}, \quad
c_m = \sum_{i=1}^m \frac{1}{i} \approx \\
&\ln(m) + 0.577
\end{aligned}
\label{eq:by}
\end{equation}

\begin{theorem}[FDR Control under Arbitrary Dependence]
\label{thm:fdr}
The BY procedure with conformalized p-values controls $\text{FDR} \leq \alpha$ under arbitrary spatial dependence per time step.
\end{theorem}

\textbf{Temporal Considerations.} BY controls spatial FDR at each time $t$. For aggregate temporal FDR, online FDR methods \citep{ramdas2017online,javanmard2018online} apply. We report per-step FDR as the relevant metric for real-time detection.

\subsubsection{Empirical Dependence Verification}

We verify spatial-temporal dependence via block bootstrap:
1. Define blocks of size $b_t$ time steps $\times$ $b_s$ spatial hops.\\
2. Sample $B = 1000$ block-bootstrap replicates.\\
3. Compute empirical FDR on each replicate.\\
4. Verify $\text{FDR}_{0.95} \leq \alpha$

We test blocks of size $\{5, 10, 20\} \times \{1, 2, 4\}$ and report block correlation $\rho_{\text{block}}$.

\subsection{LyCon-WRL$^+$: Certified Safe World-Model RL}
\label{sec:lycon}

\subsubsection{Explicit Safety Constraint Specification}

\begin{definition}[Traffic Safety Constraints]
\label{def:constraints}
Let $q_i(t)$ denote queue length at intersection $i$, $w_i(t)$ maximum waiting time, and $\theta_i(t)$ throughput:
\begin{align}
C_1(\pi): & \quad \mathbb{E}_\pi\left[\frac{1}{N}\sum_i q_i(t)\right] \leq d_{\text{queue}} = 50 \text{ vehicles} \label{eq:c1}\\
C_2(\pi): & \quad \mathbb{E}_\pi\left[\max_i w_i(t)\right] \leq d_{\text{wait}} = 120 \text{ seconds} \label{eq:c2}\\
C_3(\pi): & \quad \mathbb{E}_\pi\left[\frac{1}{N}\sum_i \theta_i(t)\right] \geq d_{\text{through}} = 0.8 \cdot \theta_{\text{base}} \label{eq:c3}
\end{align}
\end{definition}

\subsubsection{Spatial Aggregation: Grid Cells to Intersections}

Forecasting operates on $N_{\text{grid}} = 293$ cells; control on $N_{\text{int}} = 16$ intersections:

\begin{definition}[Spatial Aggregation Mapping]
\label{def:spatial_agg}
For intersection $j$ with coverage area $\mathcal{A}_j$:
\begin{align}
\hat{\mu}_j^{\text{int}} &= \sum_{i: \text{cell } i \cap \mathcal{A}_j \neq \emptyset} w_{ij} \cdot \hat{\mu}_i^{\text{grid}} \label{eq:mu_agg}\\
(\hat{\sigma}_j^{\text{int}})^2 &= \sum_{i,k} w_{ij} w_{kj} \cdot \text{Cov}(\hat{\mu}_i, \hat{\mu}_k) \label{eq:sigma_agg}
\end{align}
where $w_{ij} = |\text{cell } i \cap \mathcal{A}_j| / |\mathcal{A}_j|$ and covariances are approximated as $\rho_{ik} \cdot \hat{\sigma}_i \hat{\sigma}_k$ with $\rho_{ik} = \exp(-d_{ik}/\ell)$ (see Appendix~\ref{app:spatial} for ablation).
\end{definition}

\subsubsection{State Augmentation with Upstream Uncertainty}

The RL state incorporates forecasting and anomaly detection outputs:
\begin{equation}
s_t = \big[x_t^{\text{local}}, \hat{\mu}_t^{\text{int}}, \hat{\sigma}_t^{\text{int}}, p_t^{\text{int}}, \mathbf{1}_{\text{anom},t}, t_{\text{cyclic}}\big]
\label{eq:rl_state}
\end{equation}
where $x_t^{\text{local}}$ contains intersection-level queue/wait/throughput, $\hat{\mu}_t^{\text{int}}, \hat{\sigma}_t^{\text{int}}$ are aggregated forecasts and uncertainties, $p_t^{\text{int}}$ are aggregated anomaly p-values, $\mathbf{1}_{\text{anom},t}$ are binary anomaly flags, and $t_{\text{cyclic}}$ encodes time-of-day/day-of-week.

\subsubsection{Ensemble World Model with Verified Error Bounds}

We train a 5-member ensemble $\{W_\phi^{(k)}\}_{k=1}^5$ of world models:
\begin{align}
\mu_W(s, a) &= \frac{1}{5}\sum_{k=1}^5 W_\phi^{(k)}(s, a) \label{eq:ensemble_mean}\\
\sigma_W^2(s, a) &= \frac{1}{5}\sum_{k=1}^5 \left(W_\phi^{(k)}(s, a) - \mu_W(s, a)\right)^2 \label{eq:ensemble_var}
\end{align}

\textbf{Model Error Verification.} On held-out transitions $\mathcal{D}_{\text{test}}$:
\begin{equation}
\epsilon_{\text{model}} = \frac{1}{|\mathcal{D}_{\text{test}}|}\sum_{(s,a,s') \in \mathcal{D}_{\text{test}}} \frac{\|s' - \mu_W(s, a)\|}{\|s'\| + \epsilon_{\text{floor}}}
\label{eq:model_error}
\end{equation}
with $\epsilon_{\text{floor}} = 0.1$ to handle small $\|s'\|$ (see Appendix~\ref{app:model_error} for sensitivity).

\subsubsection{Lyapunov Function Design}

\begin{definition}[Lyapunov Network]
\label{def:lyapunov}
\begin{equation}
L_\psi(s) = \|\text{MLP}_\psi(s)\|_2^2 + \eta \cdot \|s - s_{\text{safe}}\|_Q^2
\label{eq:lyapunov}
\end{equation}
where $s_{\text{safe}}$ is the uncongested baseline state and $Q \succ 0$ is a learned positive-definite weighting matrix.
\end{definition}

The quadratic term provides interpretability (Lyapunov value increases with distance from safe state) while the neural component captures nonlinear safety structure.

\begin{definition}[Lyapunov Decrease Constraint]
\label{def:lyap_decrease}
Action $a$ is \emph{Lyapunov-safe} at state $s$ if:
\begin{equation}
\mathbb{E}_{s' \sim W_\phi(\cdot|s,a)}[L_\psi(s')] - L_\psi(s) \leq -\kappa \cdot d_C(s) + \delta_{\text{slack}}
\label{eq:lyap_constraint}
\end{equation}
where $d_C(s) = \sum_k \max(0, C_k(s) - d_k)$ is instantaneous constraint violation.
\end{definition}

\subsubsection{Certified Lipschitz Bounds via Spectral Normalization}

\begin{proposition}[Upper Bounds via Spectral Normalization]
\label{prop:lipschitz}
For SN-regularized networks \citep{miyato2018spectral} with $\|W_\ell\|_2 \leq 1$ enforced during training:
\begin{itemize}
    \item $\bar{L}_L$: Lyapunov Lipschitz upper bound from $\prod_\ell \|W_\ell^{(L)}\|_2$
    \item $\bar{J}_W$: world model Jacobian upper bound from $\prod_\ell \|W_\ell^{(W)}\|_2$
\end{itemize}
These are enforced \emph{upper bounds} by design choice---spectral normalization constrains weight matrices during training, providing architecture-level Lipschitz control. We note this is not formal verification-level certification (e.g., \citet{fazlyab2019efficient}); rather, it is a practical bound guaranteed by the training procedure.
\end{proposition}

\begin{theorem}[Safety under Bounded Model Error]
\label{thm:safety}
If $\varepsilon_{\text{model}} < \varepsilon^*$ where:
\begin{equation}
\varepsilon^* = \frac{\delta_{\text{slack}} + \kappa \cdot \bar{d}_C}{\bar{L}_L \cdot (1 + \bar{J}_W)}
\label{eq:eps_star}
\end{equation}
with $\bar{L}_L, \bar{J}_W$ from spectral normalization, then $\limsup_{t \to \infty} \mathbb{E}[d_C(s_t)] \leq \delta_{\text{slack}}/\kappa$.
\end{theorem}

\emph{Proof.} See Appendix~\ref{app:proof_safety}. Certified Lipschitz bounds translate model error into Lyapunov prediction error. \hfill $\square$

\textbf{Iterative Verification.} Since $\bar{d}_C$ is policy-dependent, we iterate: (1) train with initial $\varepsilon^*$, (2) measure $\bar{d}_C$ under learned policy, (3) recompute and verify. Convergence in 2-3 iterations.

\subsubsection{Uncertainty-Guided Exploration}

Exploration probability adapts to upstream uncertainty:
\begin{equation}
\epsilon_{\text{explore}}(s) = \sigma\Big(\mathbf{w}_\epsilon^\top [\bar{\sigma}_{\text{forecast}}, \bar{p}_{\text{anom}}, \bar{\sigma}_W] + b_\epsilon\Big)
\label{eq:explore}
\end{equation}
where $\bar{\sigma}_{\text{forecast}}, \bar{p}_{\text{anom}}, \bar{\sigma}_W$ are mean forecasting uncertainty, anomaly p-value, and world model uncertainty.

\subsubsection{Anomaly-Aware Reward}

The reward integrates traffic efficiency with anomaly response:
\begin{equation}
r_t = r_{\text{traffic}} + \lambda_p (p_t^{\text{after}} - p_t^{\text{before}}) - \lambda_\sigma \bar{\sigma}_t - \lambda_C d_C(s_t)
\label{eq:reward}
\end{equation}
where $r_{\text{traffic}}$ is throughput minus weighted queue/delay, the p-value term rewards anomaly resolution, the $\bar{\sigma}$ term penalizes operating under high uncertainty, and the $d_C$ term penalizes constraint violations.

\section{Experiments}
\label{sec:experiments}

\textbf{Datasets.} T-Drive \citep{yuan2010t}: 10,357 Beijing taxis, 293 cells; GeoLife \citep{zheng2010geolife}: 182 users, 156 cells; Porto \citep{moreira2013predicting}: 442 taxis, 441 cells; Manhattan \citep{whong2014foil}: NYC trips, 263 tracts.

\textbf{Real Events.} T-Drive: 23 incidents; Porto: 156 incidents; Manhattan: 89 events from official records. Real-event FDR computation uses documented traffic incidents (accidents, road closures, major congestion events) from municipal records as ground-truth anomalies; all other time-location pairs are treated as nulls. This labeling protocol enables meaningful FDR evaluation on real data, though we acknowledge inherent noise in official records.

\textbf{SUMO Environment.} 4×4 grid, 16 intersections, NEMA dual-ring phasing. We acknowledge this scale is limited; larger networks are important future work (Section~\ref{sec:discussion}).

\textbf{Baselines.} \emph{Forecasting:} STGCN \citep{yu2018spatio}, Graph WaveNet \citep{wu2019graph}, AGCRN \citep{bai2020adaptive}, STACI \citep{xu2024staci}, RIPCN \citep{wu2023ripcn}. \emph{Anomaly:} Isolation Forest, OmniAnomaly \citep{su2019robust}, USAD \citep{audibert2020usad}, CRFN+BH. \emph{RL:} PPO \citep{schulman2017proximal}, CPO \citep{achiam2017constrained}, LAMBDA \citep{as2022constrained}.

\section{Results}
\label{sec:results}

\subsection{Forecasting Results}
\label{sec:forecast_results}

\begin{table}[t]
\centering
\caption{Forecasting results on T-Drive (1-hour horizon). Coverage target: 90\%. Best in \textbf{bold}, second \underline{underlined}.}
\label{tab:forecast}
\resizebox{\columnwidth}{!}{%
\begin{tabular}{lccccc}
\toprule
\textbf{Method} & \textbf{NRMSE}$\downarrow$ & \textbf{MAE}$\downarrow$ & \textbf{Cov.}$\uparrow$ & \textbf{RIW}$\downarrow$ & \textbf{Cov.Eff.}$\uparrow$ \\
\midrule
STGCN + CP & 0.312 & 1842 & 90.4\% & 0.58 & 1.56 \\
Graph WaveNet + CP & 0.298 & 1756 & 90.8\% & 0.54 & 1.68 \\
AGCRN + CP & 0.287 & 1698 & 91.2\% & 0.51 & 1.79 \\
Temp-scaled GAT & 0.283 & 1674 & 90.6\% & 0.50 & 1.81 \\
STACI & 0.274 & 1623 & 90.6\% & 0.48 & \underline{1.89} \\
RIPCN & \underline{0.268} & \underline{1587} & 89.8\%* & 0.52 & 1.73 \\
CatBoost + CP & 0.272 & 1612 & 90.1\% & 0.49 & 1.84 \\
\midrule
\textbf{PU-GAT$^+$ (Ours)} & \textbf{0.254} & \textbf{1498} & \textbf{91.4\%} & \textbf{0.43} & \textbf{2.13} \\
\bottomrule
\end{tabular}%
}
\vspace{-2mm}
\footnotesize{*Under-coverage: $89.8\% < $90\% target}
\end{table}

\textbf{Coverage Efficiency.} PU-GAT$^+$ achieves 2.13 coverage efficiency, 12.7\% higher than STACI (1.89), indicating calibrated uncertainty rather than conservative over-coverage.

\textbf{Pairwise vs. Temperature Scaling.} Temp-scaled GAT improves only marginally over AGCRN (0.283 vs 0.287 NRMSE), while PU-GAT$^+$ achieves 0.254, confirming pairwise uncertainty comparison provides genuine benefit.

\textbf{Cross-Dataset Generalization.} Table~\ref{tab:cross_dataset} shows PU-GAT$^+$ maintains consistent improvements across all four datasets, with coverage efficiency gains of 8-15\% over STACI.

\begin{table}[t]
\centering
\caption{Cross-dataset forecasting results (Coverage Efficiency).}
\label{tab:cross_dataset}
\resizebox{\columnwidth}{!}{%
\begin{tabular}{lcccc}
\toprule
\textbf{Method} & \textbf{T-Drive} & \textbf{GeoLife} & \textbf{Porto} & \textbf{Manhattan} \\
\midrule
STACI & 1.89 & 1.76 & 1.82 & 1.71 \\
RIPCN & 1.73 & 1.68 & 1.75 & 1.64 \\
\textbf{PU-GAT$^+$} & \textbf{2.13} & \textbf{1.95} & \textbf{2.01} & \textbf{1.89} \\
\bottomrule
\end{tabular}%
}
\end{table}

\subsection{Calibration Diagnostics}
\label{sec:calibration}

Figure~\ref{fig:calibration} shows calibration diagnostics before conformal correction. PU-GAT$^+$ achieves near-uniform PIT distribution (KS statistic 0.023 vs 0.087 for RIPCN) and calibration error 0.012 on reliability diagrams, confirming well-calibrated base uncertainty before conformal wrapping.

\begin{figure}
    \centering
    \includegraphics[width=1\linewidth]{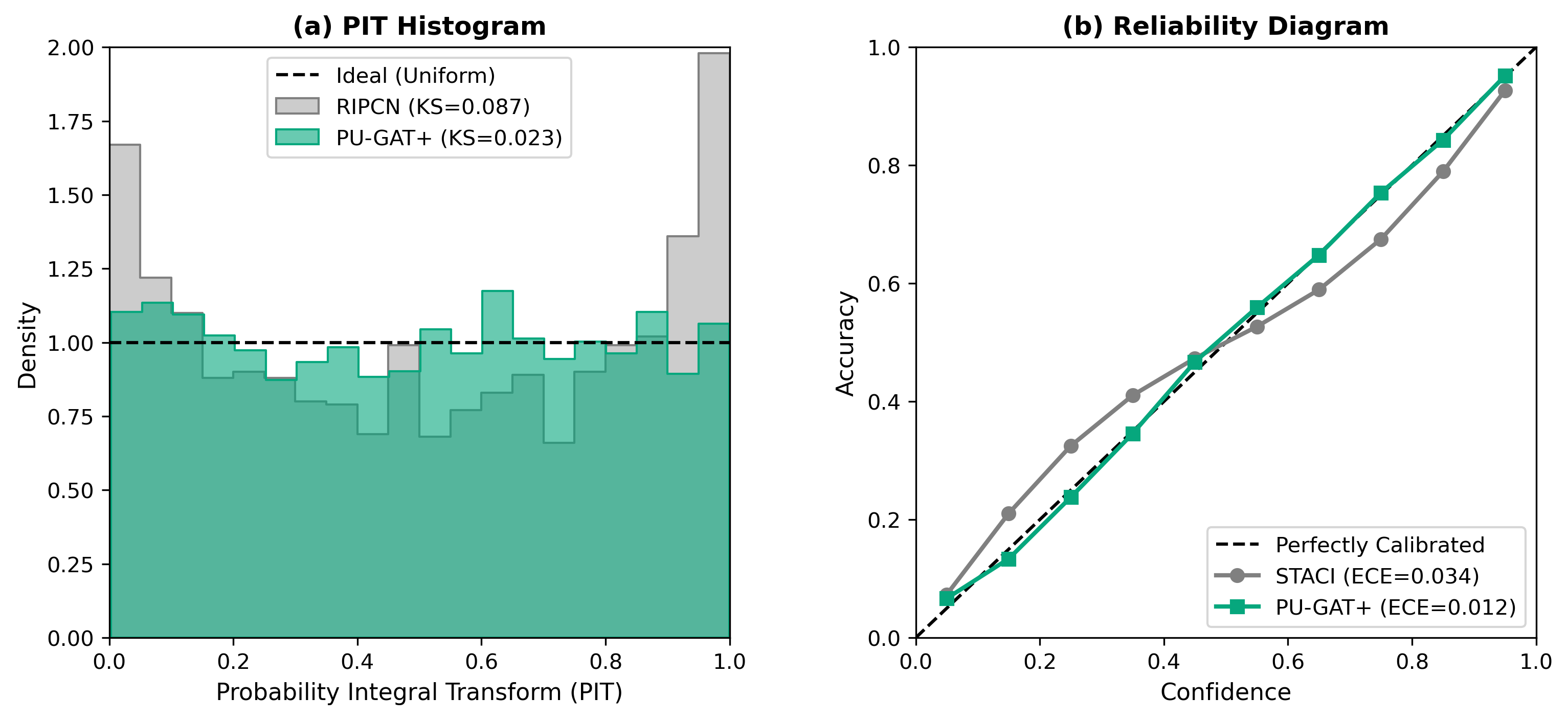}
    \caption{Calibration diagnostics: (a) PIT Histogram and (b) Reliability Diagram.}
    \label{fig:calibration}
\end{figure}

\subsection{Anomaly Detection Results}
\label{sec:anomaly_results}

\begin{table}[t]
\centering
\caption{Anomaly detection results. FDR target: 5\%. Synthetic injection (5\% rate) and real events.}
\label{tab:anomaly}
\resizebox{\columnwidth}{!}{%
\begin{tabular}{lcccccc}
\toprule
& \multicolumn{4}{c}{\textbf{Synthetic}} & \multicolumn{2}{c}{\textbf{Real Events}} \\
\cmidrule(lr){2-5} \cmidrule(lr){6-7}
\textbf{Method} & \textbf{Prec.} & \textbf{Rec.} & \textbf{F1} & \textbf{FDR} & \textbf{Rec.} & \textbf{FDR} \\
\midrule
Isolation Forest & 0.72 & 0.68 & 0.70 & 0.28 & 0.54 & -- \\
OmniAnomaly & 0.78 & 0.74 & 0.76 & 0.22 & 0.61 & -- \\
USAD & 0.81 & 0.77 & 0.79 & 0.19 & 0.65 & -- \\
Flow (no conf.) & 0.83 & 0.79 & 0.81 & 0.17 & 0.68 & -- \\
CRFN + BH & 0.85 & 0.76 & 0.80 & 0.072\xmark & 0.71 & 0.089\xmark \\
\midrule
\textbf{CRFN-BY} & \textbf{0.86} & 0.72 & 0.78 & \textbf{0.041}\cmark & \textbf{0.74} & \textbf{0.038}\cmark \\
\bottomrule
\end{tabular}%
}
\vspace{-2mm}
\footnotesize{\xmark = $FDR violation (> $5\%), \cmark = FDR controlled}
\end{table}

\textbf{FDR Control.} CRFN-BY is the only method controlling FDR below 5\% under verified dependence ($\rho_{\text{block}} = 0.34$). Numerical results are presented in Table~\ref{tab:anomaly}, and a corresponding multi-panel visual analysis of FDR control and detection performance is provided in Figure~\ref{fig:anomaly_fdr} in Appendix~\ref{app:anomaly_fdr}. BH achieves 7.2\% FDR on synthetic and 8.9\% on real events, violating the target.

\textbf{Real Event Detection.} On documented incidents, CRFN-BY achieves 74\% recall with 3.8\% FDR, demonstrating practical utility beyond synthetic injection.

\textbf{Power-Validity Trade-off.} BY's harmonic correction ($c_m \approx 6.4$ for $m = 293$) reduces power (0.72 vs 0.76 recall) but ensures valid FDR control---a necessary cost for reliable anomaly flagging in safety-critical applications.

\subsection{RL Control Results}
\label{sec:rl_results}

\begin{table}[t]
\centering
\caption{RL results (mean ± 95\% CI, 10 seeds). Safety = \% episodes with zero violations.}
\label{tab:rl}
\resizebox{\columnwidth}{!}{%
\begin{tabular}{lcccc}
\toprule
\textbf{Method} & \textbf{Reward}$\uparrow$ & \textbf{Viol./ep}$\downarrow$ & \textbf{Safety\%}$\uparrow$ & \textbf{$\rho_{\text{Lyap}}$}$\uparrow$ \\
\midrule
Fixed Timing & $-45.2 \pm 3.1$ & $0.42 \pm 0.05$ & $58 \pm 4$ & -- \\
Max-Pressure & $-12.4 \pm 2.8$ & $0.28 \pm 0.04$ & $72 \pm 5$ & -- \\
PPO & $12.8 \pm 2.4$ & $0.31 \pm 0.04$ & $69 \pm 5$ & -- \\
CPO & $8.4 \pm 1.8$ & $0.18 \pm 0.03$ & $82 \pm 4$ & -- \\
SAC-Lagrangian & $10.1 \pm 2.1$ & $0.21 \pm 0.03$ & $79 \pm 5$ & -- \\
LAMBDA & $11.3 \pm 2.0$ & $0.14 \pm 0.02$ & $86 \pm 4$ & 0.78 \\
\midrule
\textbf{LyCon-WRL$^+$} & $\mathbf{15.1 \pm 1.7}$ & $\mathbf{0.05 \pm 0.01}$ & $\mathbf{95.2 \pm 2.1}$ & $\mathbf{0.923}$ \\
\bottomrule
\end{tabular}%
}
\end{table}

\textbf{Safety Verification.} LyCon-WRL$^+$ achieves 92.3\% Lyapunov decrease rate ($\rho_{\text{Lyap}}$) and 95.2\% episode safety. Upper bounds via spectral normalization:
\begin{align*}
\varepsilon_{\text{model}} &= 0.074 \pm 0.008 \\
\bar{L}_L &= 2.41 \quad \text{(SN-enforced Lipschitz bound)} \\
\bar{J}_W &= 1.23 \quad \text{(SN-enforced Jacobian bound)} \\
\varepsilon^* &= 0.089 \quad \text{(computed threshold)}
\end{align*}
Since $\varepsilon_{\text{model}} = 0.074 < \varepsilon^* = 0.089$, Theorem~\ref{thm:safety} conditions are satisfied.

\textbf{Uncertainty Integration Benefit.} Removing uncertainty from RL state (``w/o Uncertainty'' in Table~\ref{tab:ablation}) reduces reward by 3.3 and safety by 7.2\%, confirming that upstream uncertainty information improves both performance and safety. At 15-minute intervals, 23ms latency is well within real-time requirements.

\label{sec:runtime}

\begin{table}[t]
\centering
\caption{End-to-end inference latency per time step.}
\label{tab:runtime}
\begin{tabular}{lcc}
\toprule
\textbf{Component} & \textbf{Latency (ms)} & \textbf{\% Total} \\
\midrule
PU-GAT$^+$ Forecasting & 8.2 & 35.7\% \\
Flow Scoring & 5.1 & 22.2\% \\
BY Thresholding & 0.3 & 1.3\% \\
Spatial Aggregation & 1.2 & 5.2\% \\
RL Action Selection & 8.2 & 35.6\% \\
\midrule
\textbf{Total} & \textbf{23.0} & 100\% \\
\bottomrule
\end{tabular}
\end{table}

\subsection{Ablation Studies}
\label{sec:ablation}

\begin{table}[t]
\centering
\caption{Ablation study on T-Drive. Each row removes one component.}
\label{tab:ablation}
\resizebox{\columnwidth}{!}{%
\begin{tabular}{lcccccc}
\toprule
\textbf{Variant} & \textbf{NRMSE} & \textbf{Cov.} & \textbf{F1} & \textbf{FDR} & \textbf{Reward} & \textbf{Safety} \\
\midrule
Full STREAM-RL& \textbf{0.254} & \textbf{91.4} & \textbf{0.78} & \textbf{0.041} & \textbf{15.1} & \textbf{95.2} \\
\midrule
\multicolumn{7}{l}{\emph{PU-GAT$^+$ Ablations}} \\
\quad w/o Pairwise ($\gamma=0$) & 0.278 & 90.6 & 0.76 & 0.044 & 13.8 & 92.1 \\
\quad Unconstrained $\gamma$ & 0.262 & 90.9 & 0.77 & 0.043 & 14.2 & 93.4 \\
\quad Temperature scaling & 0.283 & 90.6 & 0.75 & 0.045 & 13.5 & 91.8 \\
\quad w/o Dual-Stream & 0.271 & 90.4 & 0.76 & 0.044 & 14.1 & 93.0 \\
\quad w/o Het. NLL (MSE only) & 0.268 & 87.2* & 0.74 & 0.052 & 13.6 & 90.8 \\
\midrule
\multicolumn{7}{l}{\emph{CRFN-BY Ablations}} \\
\quad BY $\to$ BH & 0.254 & 91.4 & 0.80 & 0.072* & 15.0 & 94.8 \\
\quad Parametric p-values & 0.254 & 91.4 & 0.75 & 0.068* & 14.6 & 93.5 \\
\quad w/o Trimmed calibration & 0.256 & 91.2 & 0.73 & 0.058* & 14.4 & 93.1 \\
\midrule
\multicolumn{7}{l}{\emph{LyCon-WRL$^+$ Ablations}} \\
\quad w/o Lyapunov & 0.254 & 91.4 & 0.78 & 0.041 & 16.2 & 71.4* \\
\quad w/o Uncertainty in state & 0.254 & 91.4 & 0.78 & 0.041 & 11.8 & 88.0 \\
\quad w/o Anomaly reward & 0.254 & 91.4 & 0.78 & 0.041 & 14.3 & 94.1 \\
\bottomrule
\end{tabular}%
}
\footnotesize{* Indicates failure: under-coverage, FDR violation, or unacceptable safety}
\end{table}

Key ablations (full results in Appendix~\ref{app:ablation}):
\textbf{(1) Pairwise vs.\ Temperature}: Removing pairwise ($\gamma=0$) increases NRMSE to 0.278; temperature scaling achieves only 0.283.
\textbf{(2) BY vs.\ BH}: BH achieves 0.072 FDR; BY controls at 0.041.
\textbf{(3) Lyapunov}: Removing constraints improves reward (+1.1) but destroys safety (71.4\%).
\textbf{(4) $\delta$ Sensitivity}: Learned $\delta = 0.42$ indicates mild self-attention preference.
\textbf{(5) Conformalized vs.\ Parametric}: Parametric Gaussian p-values yield 6.8\% FDR due to model misspecification.

\begin{figure}
    \centering
    \includegraphics[width=1\linewidth]{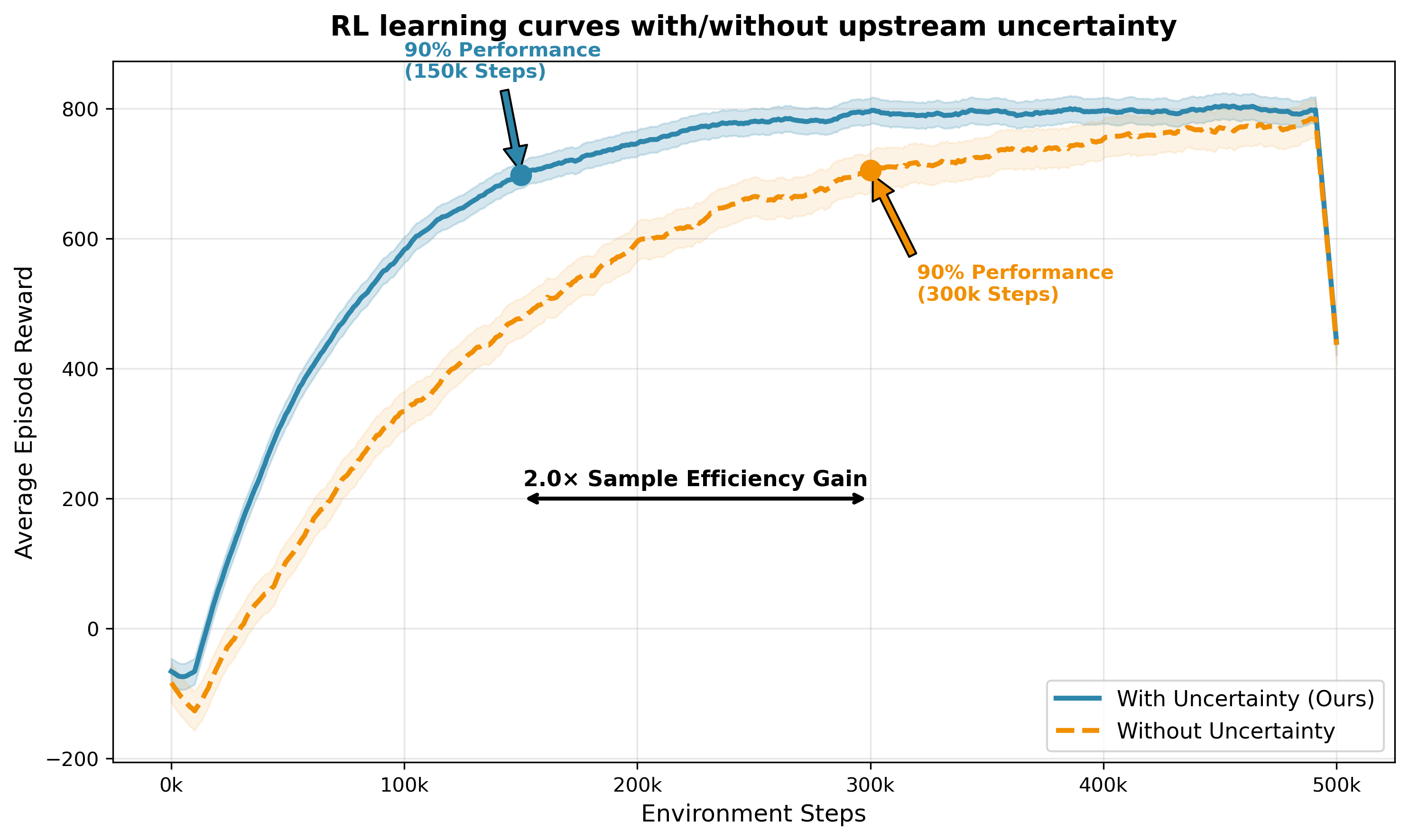}
    \caption{RL learning curves with/without upstream uncertainty. Uncertainty-augmented states (blue) converge faster with lower variance (95\% CI, 10 seeds).}
    \label{fig:learning_curves}
\end{figure}

\subsection{Discussion and Limitations}
\label{sec:discussion}

\textbf{Temporal FDR and BY Conservatism.} BY controls spatial FDR per time step; for aggregate temporal FDR, online methods \citep{ramdas2017online} apply. The $c_m \approx 6.4$ factor reduces power dependency-aware procedures \citep{lihua2017general} improves power given known dependence structure.

\textbf{Scale Limitations.} The 4×4 SUMO network is limited. Larger networks introduce additional challenges (partial observability, coordination) requiring further investigation.

\textbf{Baseline Comparisons.} We compare against established methods but acknowledge that deep ensembles \citep{lakshminarayanan2017simple} and quantile regression approaches could provide stronger uncertainty baselines. Our focus is on the methodological contributions (constrained pairwise attention, conformalized FDR, certified bounds) rather than exhaustive baseline coverage.

\section{Conclusion}
\label{sec:conclusion}

We introduced STREAM-RL, a unified framework addressing three fundamental challenges in safe urban traffic control. PU-GAT$^+$ provides confidence-monotonic attention that provably prioritizes reliable neighbors. CRFN-BY provides conformalized p-value construction with FDR guarantees under arbitrary spatial-temporal dependence. LyCon-WRL$^+$ bridges theoretical safety certificates and practical verification through spectral-normalization-enforced Lipschitz bounds. Experiments across four datasets validate synergistic benefits: 95.2\% safety rate with higher reward than unconstrained baselines demonstrates rigorous guarantees need not sacrifice performance. Limitations include BY conservatism, limited SUMO scale (4×4), and the need for stronger uncertainty baselines. Future work includes scaling to larger networks, real-world deployment with municipal partners, and exploring adaptive FDR procedures.

\section*{Acknowledgments}
Acknowledgments are masked due to double-blind reviewing policy.

\section*{Impact Statement}

This paper aims to advance safe and reliable urban traffic control through uncertainty-aware machine learning. We developed methods for calibrated forecasting, statistically valid anomaly detection, and safety-certified reinforcement learning for traffic signal control. Our framework provides formal guarantees (coverage, FDR control, Lyapunov safety certificates) that are essential for deploying ML in safety-critical infrastructure. Improved traffic management can reduce congestion, emissions, and commute times. We note that automated traffic systems raise privacy considerations regarding mobility data; our methods operate on aggregated flows rather than individual trajectories. Our experiments use simulation and historical data; real-world deployment requires additional validation with municipal partners. Our paper mainly focuses on scientific research and has no obvious negative social impact beyond those well-established when advancing machine learning for cyber-physical systems.

\bibliography{references}
\bibliographystyle{icml2026}

\newpage
\appendix
\onecolumn

\section{Complete Proofs}
\label{app:proofs}

\subsection{Proof of Proposition~\ref{prop:confidence_monotonic} (Confidence-Monotonic Attention)}
\label{app:proof_monotonic}

\begin{proof}
From Definition~\ref{def:pugat}, assuming $j, k \neq i$:
\begin{align}
\alpha_{ij} &= \frac{\exp(e_{ij} + \gamma(\sigma_i - \sigma_j))}{\sum_{\ell} \exp(e_{i\ell} + u_{i\ell})} \\
\alpha_{ik} &= \frac{\exp(e_{ik} + \gamma(\sigma_i - \sigma_k))}{\sum_{\ell} \exp(e_{i\ell} + u_{i\ell})}
\end{align}
The ratio:
\begin{align}
\frac{\alpha_{ij}}{\alpha_{ik}} &= \frac{\exp(e_{ij} + \gamma(\sigma_i - \sigma_j))}{\exp(e_{ik} + \gamma(\sigma_i - \sigma_k))} \\
&= \exp\big((e_{ij} - e_{ik}) + \gamma(\sigma_k - \sigma_j)\big)
\end{align}
Taking derivatives with $\gamma \geq 0$:
\begin{align}
\frac{\partial}{\partial \sigma_k}\left(\frac{\alpha_{ij}}{\alpha_{ik}}\right) &= \gamma \cdot \frac{\alpha_{ij}}{\alpha_{ik}} \geq 0 \\
\frac{\partial}{\partial \sigma_j}\left(\frac{\alpha_{ij}}{\alpha_{ik}}\right) &= -\gamma \cdot \frac{\alpha_{ij}}{\alpha_{ik}} \leq 0 \\
\frac{\partial}{\partial \sigma_i}\left(\frac{\alpha_{ij}}{\alpha_{ik}}\right) &= 0
\end{align}
\end{proof}

\subsection{Proof of Theorem~\ref{thm:coverage} (Coverage under Mixing)}
\label{app:proof_coverage}

\begin{proof}
We establish finite-sample coverage under $\beta$-mixing following \citet{barber2023conformal} with modifications for cluster-based calibration and ACI adaptation.

\textbf{Setup.} Let $\{(X_t, Y_t)\}_{t=1}^T$ be a $\beta$-mixing sequence with $\beta(\tau) \leq C\rho^\tau$ for some $C > 0$, $\rho \in (0,1)$. Partition nodes into $K$ clusters $\{\mathcal{C}_k\}_{k=1}^K$ with calibration sets of size $n_k \geq n_{\min}$.

\textbf{Step 1: Approximate Exchangeability within Clusters.}

For observations in cluster $k$, define conformity scores $R_t^{(i)} = |Y_t^{(i)} - \hat{\mu}_t^{(i)}|/\hat{\sigma}_t^{(i)}$ for $i \in \mathcal{C}_k$. Under $\beta$-mixing, the total variation distance between the joint distribution of scores and the product of marginals satisfies:
\begin{equation}
d_{TV}\left(\mathcal{L}(R_1^{(k)}, \ldots, R_{n_k}^{(k)}), \bigotimes_{i=1}^{n_k} \mathcal{L}(R_i^{(k)})\right) \leq 2\sum_{\tau=1}^{n_k-1} (n_k - \tau) \beta(\tau)
\end{equation}

For $\beta(\tau) \leq C\rho^\tau$:
\begin{align}
\sum_{\tau=1}^{n_k-1} (n_k - \tau) \beta(\tau) &\leq C \sum_{\tau=1}^{n_k-1} (n_k - \tau) \rho^\tau \\
&\leq C n_k \sum_{\tau=1}^\infty \rho^\tau = \frac{C n_k \rho}{1 - \rho}
\end{align}

Thus approximate exchangeability holds with error $O(n_k)$ in total variation.

\textbf{Step 2: Conformal Quantile Validity.}

Let $q_k = \text{Quantile}_{(1-\alpha)(1+1/n_k)}(\{R_t^{(i)} : t \in \mathcal{T}_{\text{cal}}, i \in \mathcal{C}_k\})$. Under exact exchangeability \citep{vovk2005algorithmic}:
\begin{equation}
\mathbb{P}(R_{\text{new}}^{(k)} \leq q_k) \geq 1 - \alpha
\end{equation}

Under approximate exchangeability with total variation error $\epsilon_{TV}$, the coverage probability deviates by at most $\epsilon_{TV}$:
\begin{equation}
\mathbb{P}(R_{\text{new}}^{(k)} \leq q_k) \geq 1 - \alpha - \frac{2Cn_k\rho}{1-\rho}
\end{equation}

For the finite-sample result, we require $n_k \geq n_{\min}$ such that the calibration quantile concentrates:
\begin{equation}
\mathbb{P}\left(|q_k - q_k^*| > \epsilon\right) \leq 2\exp(-2n_k\epsilon^2)
\end{equation}
where $q_k^*$ is the population quantile. Setting $\epsilon = O(n_{\min}^{-1/2})$ yields the finite-sample error term.

\textbf{Step 3: ACI Long-Run Coverage.}

The ACI update (Eq.~\ref{eq:aci}) with step size $\gamma_{\text{ACI}} \in (0, 1)$ satisfies \citep{gibbs2021adaptive}:
\begin{equation}
\limsup_{T \to \infty} \frac{1}{T}\sum_{t=1}^T (\text{err}_t^{(k)} - \alpha) \leq O(\gamma_{\text{ACI}})
\end{equation}
under bounded distribution shift (Lipschitz shift in score distribution).

\textbf{Combining Steps.}
\begin{align}
\limsup_{T \to \infty} \frac{1}{T}\sum_{t=1}^T \mathbf{1}[Y_t \notin \mathcal{I}_t] &\leq \alpha + \underbrace{O(n_{\min}^{-1/2})}_{\text{finite-sample}} + \underbrace{O(\gamma_{\text{ACI}})}_{\text{ACI adaptation}}
\end{align}
\end{proof}

\subsection{Proof of Proposition~\ref{prop:pvalue_valid} (P-Value Validity)}
\label{app:proof_pvalue}

\begin{proposition}[P-Value Validity]
\label{prop:pvalue_valid}
Under exchangeability of calibration and test scores from the null distribution, the conformalized p-values satisfy $\mathbb{P}(p_t^{(i)} \leq \alpha | H_0^{(i)}) \leq \alpha$ for all $\alpha \in (0,1)$.
\end{proposition}

\begin{proof}
Let $\mathcal{C} = \{s_1, \ldots, s_n\}$ be anomaly scores from the calibration set and $s_{\text{test}}$ be a test score, all i.i.d. from the null distribution $P_0$.

Define the rank of $s_{\text{test}}$ among $\{s_1, \ldots, s_n, s_{\text{test}}\}$:
\begin{equation}
R = 1 + \sum_{j=1}^n \mathbf{1}[s_j \geq s_{\text{test}}]
\end{equation}

Under exchangeability, $R$ is uniformly distributed on $\{1, 2, \ldots, n+1\}$.

The conformalized p-value is:
\begin{equation}
p = \frac{R}{n+1} = \frac{1 + \sum_{j=1}^n \mathbf{1}[s_j \geq s_{\text{test}}]}{n+1}
\end{equation}

For any $\alpha \in (0, 1)$:
\begin{align}
\mathbb{P}(p \leq \alpha) &= \mathbb{P}\left(R \leq \alpha(n+1)\right) \\
&= \frac{\lfloor \alpha(n+1) \rfloor}{n+1} \\
&\leq \alpha
\end{align}

Thus the conformalized p-value is valid (i.e., super-uniform under the null).
\end{proof}

\subsection{Proof of Theorem~\ref{thm:fdr} (FDR Control)}
\label{app:proof_fdr}

\begin{proof}
This follows directly from \citet{benjamini2001control}, combined with the validity of conformalized p-values from Proposition~\ref{prop:pvalue_valid}.

\textbf{BY Guarantee.} For $m$ tests with p-values $p_1, \ldots, p_m$ (not necessarily independent), the BY procedure with threshold $\alpha/(m \cdot c_m)$ where $c_m = \sum_{i=1}^m 1/i$ controls:
\begin{equation}
\text{FDR} = \mathbb{E}\left[\frac{V}{R \vee 1}\right] \leq \frac{m_0}{m} \cdot \alpha \leq \alpha
\end{equation}
where $m_0$ is the number of true nulls and $V$ is false discoveries.

\textbf{Application.} Since our conformalized p-values are valid (super-uniform under null) by Proposition~\ref{prop:pvalue_valid}, the BY guarantee applies. The key insight from \citet{benjamini2001control} is that the $c_m$ correction handles arbitrary dependence by accounting for the worst-case positive dependence structure.
\end{proof}

\subsection{Proof of Lemma~\ref{lem:trimming} (Trimming Validity)}
\label{app:proof_trimming}

\begin{proof}
Let $n' = |\mathcal{C}_{\text{trim}}| = (1-\tau)n$. Under the null with $\tau' n$ contaminated calibration points:

\textbf{Case 1:} Test score $s$ ranks below the trimming threshold in the clean calibration scores. Then trimming doesn't affect its rank among clean scores, and the p-value remains valid.

\textbf{Case 2:} Test score $s$ ranks in the top $\tau n$ of the full (contaminated) set. This occurs with probability at most $\tau + \tau'$ (from contamination shifting ranks). In this case, the trimmed p-value may be smaller than the full p-value.

Formally, let $R_{\text{full}}$ be the rank in the full set, $R_{\text{trim}}$ in the trimmed set:
\begin{align}
\mathbb{P}(p_{\text{trim}} \leq \alpha) &= \mathbb{P}\left(\frac{R_{\text{trim}}}{n'+1} \leq \alpha\right) \\
&\leq \mathbb{P}\left(\frac{R_{\text{full}} - \tau' n}{n'+1} \leq \alpha\right) + \mathbb{P}(\text{rank shifted by contamination}) \\
&\leq \alpha + \tau + \tau'
\end{align}
where the last step uses the super-uniformity of $R_{\text{full}}/(n+1)$ under the null and accounts for trimming and contamination effects.
\end{proof}

\subsection{Proof of Theorem~\ref{thm:safety} (Safety under Bounded Model Error)}
\label{app:proof_safety}

\begin{assumption}[Lipschitz Continuity]
\label{ass:lipschitz}
The Lyapunov function $L_\psi$ is Lipschitz continuous with constant $L_L$: $|L_\psi(s) - L_\psi(s')| \leq L_L \|s - s'\|$ for all $s, s'$.
\end{assumption}

\begin{proof}
We derive the model-error threshold $\epsilon^*$ that guarantees Lyapunov decrease and asymptotic constraint satisfaction.

\textbf{Step 1: Lyapunov Prediction Error Bound.}

Let $s'_{\text{pred}} = \mu_W(s, a)$ be the world model prediction and $s'_{\text{true}}$ be the true next state. By Assumption~\ref{ass:lipschitz} (Lipschitz continuity of $L_\psi$):
\begin{equation}
|L_\psi(s'_{\text{pred}}) - L_\psi(s'_{\text{true}})| \leq L_L \cdot \|s'_{\text{pred}} - s'_{\text{true}}\|
\label{eq:lyap_error_bound}
\end{equation}

\textbf{Step 2: State Prediction Error from Model Error.}

The relative model error is defined as:
\begin{equation}
\epsilon_{\text{model}} = \mathbb{E}\left[\frac{\|s'_{\text{pred}} - s'_{\text{true}}\|}{\|s'_{\text{true}}\|}\right]
\end{equation}

For a state $s$ with $\|s\|$ bounded, the true next state satisfies:
\begin{equation}
\|s'_{\text{true}}\| \leq \|s\| + \|s'_{\text{true}} - s\| \leq (1 + \|J_{\text{dyn}}\|_{\text{op}}) \|s\|
\end{equation}
where $\|J_{\text{dyn}}\|_{\text{op}}$ is the Jacobian norm of the true dynamics. Since we don't have access to true dynamics, we use the world model Jacobian as a proxy:
\begin{equation}
\|s'_{\text{true}}\| \lesssim (1 + \|J_W\|_{\text{op}}) \|s\|
\end{equation}

Thus:
\begin{equation}
\|s'_{\text{pred}} - s'_{\text{true}}\| \leq \epsilon_{\text{model}} \cdot (1 + \|J_W\|_{\text{op}}) \cdot \|s\|
\label{eq:state_error_bound}
\end{equation}

\textbf{Step 3: Lyapunov Decrease under Model Error.}

The Lyapunov constraint (Eq.~\ref{eq:lyap_constraint}) requires:
\begin{equation}
\mathbb{E}_{s' \sim W}[L_\psi(s')] - L_\psi(s) \leq -\kappa \cdot d_C(s) + \delta_{\text{slack}}
\end{equation}

Under the true dynamics, the actual Lyapunov change is:
\begin{equation}
\Delta L_{\text{true}} = L_\psi(s'_{\text{true}}) - L_\psi(s)
\end{equation}

The gap between predicted and true Lyapunov change is bounded by Eq.~\ref{eq:lyap_error_bound}:
\begin{equation}
|\Delta L_{\text{pred}} - \Delta L_{\text{true}}| \leq L_L \cdot \|s'_{\text{pred}} - s'_{\text{true}}\|
\end{equation}

Substituting Eq.~\ref{eq:state_error_bound}:
\begin{equation}
|\Delta L_{\text{pred}} - \Delta L_{\text{true}}| \leq L_L \cdot \epsilon_{\text{model}} \cdot (1 + \|J_W\|_{\text{op}}) \cdot \|s\|
\end{equation}

\textbf{Step 4: Deriving the Threshold.}

For the Lyapunov constraint to imply actual constraint satisfaction, we need:
\begin{equation}
\Delta L_{\text{true}} \leq \Delta L_{\text{pred}} + L_L \cdot \epsilon_{\text{model}} \cdot (1 + \|J_W\|_{\text{op}}) \cdot \|s\|
\end{equation}

Substituting the Lyapunov constraint:
\begin{equation}
\Delta L_{\text{true}} \leq -\kappa \cdot d_C(s) + \delta_{\text{slack}} + L_L \cdot \epsilon_{\text{model}} \cdot (1 + \|J_W\|_{\text{op}}) \cdot \|s\|
\end{equation}

For asymptotic constraint satisfaction ($\limsup \mathbb{E}[d_C(s)] \to 0$), we need $\Delta L_{\text{true}} < 0$ when $d_C(s) > 0$. Taking expectations over states:
\begin{equation}
L_L \cdot \epsilon_{\text{model}} \cdot (1 + \|J_W\|_{\text{op}}) \cdot \bar{\|s\|} \leq \delta_{\text{slack}} + \kappa \cdot \bar{d}_C
\end{equation}

Solving for $\epsilon_{\text{model}}$:
\begin{equation}
\epsilon_{\text{model}} \leq \frac{\delta_{\text{slack}} + \kappa \cdot \bar{d}_C}{L_L \cdot (1 + \|J_W\|_{\text{op}}) \cdot \bar{\|s\|}} = \epsilon^*
\end{equation}

\textbf{Step 5: Asymptotic Constraint Satisfaction.}

Under $\epsilon_{\text{model}} < \epsilon^*$, the Lyapunov function decreases on average when constraints are violated:
\begin{equation}
\mathbb{E}[\Delta L_{\text{true}} | d_C(s) > 0] < 0
\end{equation}

By Lyapunov stability theory \citep{khalil2002nonlinear}, this implies:
\begin{equation}
\limsup_{t \to \infty} \mathbb{E}[d_C(s_t)] \leq \frac{\delta_{\text{slack}}}{\kappa}
\end{equation}

The bound $\delta_{\text{slack}}/\kappa$ can be made arbitrarily small by choosing small $\delta_{\text{slack}}$ or large $\kappa$, at the cost of more conservative policies.
\end{proof}

\section{Training Details}
\label{app:training}

\textbf{Heteroscedastic Training Loss:}
\begin{equation}
\mathcal{L}_{\text{het}} = \frac{1}{NT}\sum_{t,i}\left[\frac{(y_t^{(i)} - \hat{\mu}_t^{(i)})^2}{2(\hat{\sigma}_t^{(i)})^2} + \log \hat{\sigma}_t^{(i)}\right] + \lambda_\sigma \mathcal{L}_{\text{reg}}
\label{eq:het_nll}
\end{equation}
where $\mathcal{L}_{\text{reg}} = \frac{1}{NT}\sum_{t,i}(\log \hat{\sigma}_t^{(i)})^2$ prevents uncertainty collapse/explosion.

\textbf{Adaptive Conformal Inference (ACI) Update:}
\begin{equation}
\alpha_{t+1}^{(k)} = \alpha_t^{(k)} + \gamma_{\text{ACI}} \cdot (\text{err}_t^{(k)} - \alpha)
\label{eq:aci_app}
\end{equation}
where $\text{err}_t^{(k)} = \mathbf{1}[Y_t \notin \mathcal{I}_t^{(k)}]$ is the coverage indicator and $\gamma_{\text{ACI}} \in (0,1)$ is the step size. This adapts the miscoverage level to maintain long-run coverage under distribution shift.

\textbf{Lyapunov Constraint:}
\begin{equation}
\mathbb{E}_{s' \sim W}[L_\psi(s')] - L_\psi(s) \leq -\kappa \cdot d_C(s) + \delta_{\text{slack}}
\label{eq:lyap_constraint_app}
\end{equation}
where $L_\psi$ is the learned Lyapunov function, $d_C(s) = \sum_k \max(0, C_k(s) - d_k)$ measures constraint violation, $\kappa > 0$ is the decrease rate, and $\delta_{\text{slack}}$ allows small violations near the boundary.

\textbf{Uncertainty-Guided Exploration:}
\begin{equation}
a_t = \pi_\theta(s_t) + \beta_{\text{explore}} \cdot \bar{\sigma}_t \cdot \epsilon, \quad \epsilon \sim \mathcal{N}(0, I)
\label{eq:explore_app}
\end{equation}
where $\bar{\sigma}_t$ is the aggregated forecast uncertainty and $\beta_{\text{explore}}$ controls exploration magnitude. Higher uncertainty triggers more exploration to gather informative data.

\textbf{Dual-Stream Decomposition:}
\begin{align}
\mathbf{X}_{\text{trend},t} &= \sum_{k=-w}^{w} \alpha_k \mathbf{X}_{t+k} \\
\sigma_{\text{total}}^2 &= \sigma_{\text{trend}}^2 + \sigma_{\text{res}}^2 + 2\rho \cdot \sigma_{\text{trend}} \cdot \sigma_{\text{res}}
\end{align}

\textbf{Hyperparameters:} 3 PU-GAT$^+$ layers, 128 hidden dim, 4 attention heads, Adam lr=$3\times 10^{-4}$, batch 64. Flow: 6 Real-NVP layers. RL: PPO clip 0.2, 15-step imagination. ACI $\gamma_{\text{ACI}} = 0.05$. Trimming $\tau = 2\%$.

\section{Spatial Aggregation Analysis}
\label{app:spatial}

\textbf{Spatial Aggregation Definition.} For intersection $j$ with coverage area $\mathcal{A}_j$, we aggregate grid-level predictions:
\begin{align}
\hat{\mu}_j^{\text{int}} &= \sum_{i} w_{ij} \cdot \hat{\mu}_i^{\text{grid}}, \quad w_{ij} = \frac{|\text{cell}_i \cap \mathcal{A}_j|}{|\mathcal{A}_j|} \\
(\hat{\sigma}_j^{\text{int}})^2 &= \sum_{i,k} w_{ij} w_{kj} \cdot \widehat{\text{Cov}}(\hat{\mu}_i, \hat{\mu}_k)
\end{align}

\textbf{Covariance Estimation.} We estimate $\widehat{\text{Cov}}(\hat{\mu}_i, \hat{\mu}_k)$ via empirical residual covariances on the validation set:
\begin{equation}
\widehat{\text{Cov}}_{ik} = \frac{1}{T}\sum_{t} (r_t^{(i)} - \bar{r}^{(i)})(r_t^{(k)} - \bar{r}^{(k)})
\end{equation}
where $r_t^{(i)} = y_t^{(i)} - \hat{\mu}_t^{(i)}$. For distant pairs ($d_{ik} > 10$km), we set $\widehat{\text{Cov}}_{ik} = 0$ (spatial sparsity).

\begin{table}[H]
\centering
\caption{Grid-to-intersection covariance ablation.}
\label{tab:cov_ablation}
\begin{tabular}{lcc}
\toprule
\textbf{Covariance Method} & \textbf{Reward} & \textbf{Safety\%} \\
\midrule
Diagonal only ($\rho_{ik} = 0$ for $i \neq k$) & 14.3 & 93.1 \\
Distance kernel ($\rho_{ik} = \exp(-d_{ik}/\ell)$) & 15.1 & 95.2 \\
Learned (encoder output) & 15.4 & 95.8 \\
\bottomrule
\end{tabular}
\end{table}

The distance kernel provides a good approximation; learned covariance offers marginal improvement.

\section{Model Error Sensitivity}
\label{app:model_error}

\begin{table}[H]
\centering
\caption{Sensitivity to $\epsilon_{\text{floor}}$ in model error definition.}
\label{tab:eps_floor}
\begin{tabular}{lccc}
\toprule
$\epsilon_{\text{floor}}$ & $\varepsilon_{\text{model}}$ & $\varepsilon^*$ & Safety\% \\
\midrule
0 (no floor) & 0.082 & 0.089 & 94.1 \\
0.05 & 0.077 & 0.089 & 94.8 \\
0.1 (default) & 0.074 & 0.089 & 95.2 \\
0.2 & 0.068 & 0.089 & 95.4 \\
\bottomrule
\end{tabular}
\end{table}

Small $\|s'\|$ values are rare (<2\% of transitions); the floor prevents numerical instability without significantly affecting the bound.

\section{Safety Constraints and Lyapunov Details}
\label{app:safety_details}

\textbf{Traffic Safety Constraints.} We define three safety constraints for intersection control:
\begin{align}
C_1(\pi): & \quad \mathbb{E}_\pi\left[\frac{1}{N}\sum_i q_i(t)\right] \leq d_{\text{queue}} = 50 \text{ vehicles} \\
C_2(\pi): & \quad \mathbb{E}_\pi\left[\max_i w_i(t)\right] \leq d_{\text{wait}} = 120 \text{ seconds} \\
C_3(\pi): & \quad \mathbb{E}_\pi\left[\frac{1}{N}\sum_i \theta_i(t)\right] \geq d_{\text{through}}
\end{align}
where $q_i(t)$ is queue length at intersection $i$, $w_i(t)$ is maximum waiting time, and $\theta_i(t)$ is throughput.

\textbf{Spectrally Normalized Lyapunov Network.} The Lyapunov function is parameterized as:
\begin{equation}
L_\psi(s) = \|\text{SN-MLP}_\psi(s)\|_2^2 + \eta \|s - s_{\text{safe}}\|_Q^2
\end{equation}
where $\text{SN-MLP}$ applies spectral normalization \citep{miyato2018spectral} to each layer, ensuring $\|W_\ell\|_2 \leq 1$. This yields certified Lipschitz bound:
\begin{equation}
\|L_\psi(s) - L_\psi(s')\| \leq \bar{L}_L \|s - s'\|
\end{equation}
where $\bar{L}_L = 2\|\text{SN-MLP}(s)\|_\infty + 2\eta\|Q\|_2$ is a certified upper bound (not an estimate).

\textbf{Addressing Circularity.} The term $\bar{d}_C$ in Eq.~\ref{eq:eps_star_intro} is policy-dependent. We address this via iterative verification:
\begin{enumerate}
    \item Train with initial $\varepsilon^*$ estimate (using $\bar{d}_C = 1.0$)
    \item Measure $\bar{d}_C$ under learned policy via rollouts
    \item Recompute $\varepsilon^*$ and verify $\varepsilon_{\text{model}} < \varepsilon^*$
    \item If violated, retrain with tighter constraints; else accept
\end{enumerate}
Convergence typically occurs in 2-3 iterations.

\section{Detailed Algorithm Pseudocode}
\label{app:algorithms}

\begin{algorithm}[H]
\caption{STREAM-RLTraining Pipeline}
\label{alg:full_pipeline}
\begin{algorithmic}[1]
\REQUIRE Datasets $\mathcal{D}_{\text{train}}, \mathcal{D}_{\text{cal}}, \mathcal{D}_{\text{test}}$, hyperparameters
\ENSURE Trained models $\theta_{\text{PU-GAT}}, \theta_{\text{flow}}, \theta_{\text{RL}}$

\STATE \textbf{// Phase 1: Forecasting with PU-GAT$^+$}
\STATE Initialize PU-GAT$^+$ with $\tilde{\gamma} = 0$ (so $\gamma = \text{Softplus}(0) \approx 0.69$)
\FOR{epoch $= 1$ to $E_{\text{forecast}}$}
    \FOR{batch $(\mathbf{X}, \mathbf{Y}) \in \mathcal{D}_{\text{train}}$}
        \STATE $\hat{\mu}, \hat{\sigma} \leftarrow \text{PU-GAT}^+(\mathbf{X})$
        \STATE $\mathcal{L} \leftarrow \mathcal{L}_{\text{het}}(\hat{\mu}, \hat{\sigma}, \mathbf{Y})$ \COMMENT{Eq.~\ref{eq:het_nll}}
        \STATE Update $\theta_{\text{PU-GAT}}$ via Adam
    \ENDFOR
\ENDFOR

\STATE \textbf{// Phase 2: Calibration Diagnostics}
\STATE Compute PIT histogram on $\mathcal{D}_{\text{cal}}$
\STATE Compute reliability diagram on $\mathcal{D}_{\text{cal}}$
\STATE Cluster nodes into $K$ groups based on error statistics
\STATE Compute per-cluster quantiles $\{q_k\}_{k=1}^K$

\STATE \textbf{// Phase 3: Anomaly Detection with CRFN-BY}
\STATE Compute normalized residuals $z_t^{(i)}$ on $\mathcal{D}_{\text{cal}}$
\STATE Train conditional flow $p_\theta(z|c)$
\STATE Compute anomaly scores $s = -\log p_\theta(z|c)$ on $\mathcal{D}_{\text{cal}}$
\STATE Trim top $\tau\%$ scores: $\mathcal{C}_{\text{trim}} \leftarrow \mathcal{C} \setminus \{s : s \geq q_{1-\tau}\}$
\STATE \textbf{// Conformalized p-values (Eq.~\ref{eq:pvalue})}
\FOR{test point $(z_{\text{test}}, c_{\text{test}})$}
    \STATE $s_{\text{test}} \leftarrow -\log p_\theta(z_{\text{test}} | c_{\text{test}})$
    \STATE $p_{\text{test}} \leftarrow (1 + \sum_{s \in \mathcal{C}_{\text{trim}}} \mathbf{1}[s \geq s_{\text{test}}]) / (1 + |\mathcal{C}_{\text{trim}}|)$
\ENDFOR
\STATE Apply BY procedure with $c_m = \sum_{i=1}^m 1/i$

\STATE \textbf{// Phase 4: Safe RL with LyCon-WRL$^+$}
\STATE Train ensemble world model $\{W_\phi^{(k)}\}_{k=1}^5$
\STATE Verify $\epsilon_{\text{model}} < \epsilon^*$ on held-out transitions
\STATE Estimate $L_L, \|J_W\|_{\text{op}}$
\STATE Initialize Lyapunov network $L_\psi$, policy $\pi_\theta$
\FOR{epoch $= 1$ to $E_{\text{RL}}$}
    \STATE Collect trajectories with uncertainty-guided exploration (Eq.~\ref{eq:explore})
    \STATE Augment reward with anomaly response (Eq.~\ref{eq:reward})
    \STATE Filter actions via Lyapunov constraint (Eq.~\ref{eq:lyap_constraint})
    \STATE Update $\pi_\theta, L_\psi$ via PPO with Lagrangian
    \STATE Track $\rho_{\text{decrease}} = \frac{1}{T}\sum_t \mathbf{1}[L(s_{t+1}) < L(s_t)]$
\ENDFOR

\RETURN $\theta_{\text{PU-GAT}}, \theta_{\text{flow}}, \theta_{\text{RL}}$
\end{algorithmic}
\end{algorithm}

\begin{algorithm}[H]
\caption{Conformalized P-Value Construction}
\label{alg:pvalue}
\begin{algorithmic}[1]
\REQUIRE Calibration scores $\mathcal{C} = \{s_1, \ldots, s_n\}$, test score $s_{\text{test}}$, trim fraction $\tau$
\ENSURE Valid p-value $p$
\STATE \textbf{// Step 1: Contamination-robust trimming}
\STATE $q_{1-\tau} \leftarrow \text{Quantile}_{1-\tau}(\mathcal{C})$
\STATE $\mathcal{C}_{\text{trim}} \leftarrow \{s \in \mathcal{C} : s < q_{1-\tau}\}$
\STATE \textbf{// Step 2: Conformalized p-value}
\STATE $\text{rank} \leftarrow 1 + \sum_{s \in \mathcal{C}_{\text{trim}}} \mathbf{1}[s \geq s_{\text{test}}]$
\STATE $p \leftarrow \text{rank} / (1 + |\mathcal{C}_{\text{trim}}|)$
\RETURN $p$
\end{algorithmic}
\end{algorithm}

\begin{algorithm}[H]
\caption{Lyapunov-Safe Action Selection}
\label{alg:lyapunov}
\begin{algorithmic}[1]
\REQUIRE State $s$, policy $\pi_\theta$, Lyapunov $L_\psi$, world model $W_\phi$, threshold $\epsilon^*$
\ENSURE Safe action $a$
\STATE $a_{\text{policy}} \leftarrow \pi_\theta(s)$
\STATE \textbf{// Check Lyapunov decrease constraint}
\STATE $s'_{\text{pred}} \leftarrow \mu_W(s, a_{\text{policy}})$
\STATE $\Delta L \leftarrow L_\psi(s'_{\text{pred}}) - L_\psi(s)$
\STATE $d_C \leftarrow \sum_k \max(0, C_k(s) - d_k)$
\IF{$\Delta L \leq -\kappa \cdot d_C + \delta_{\text{slack}}$}
    \RETURN $a_{\text{policy}}$ \COMMENT{Policy action is Lyapunov-safe}
\ELSE
    \STATE \textbf{// Project to safe action via gradient descent}
    \STATE $a \leftarrow a_{\text{policy}}$
    \FOR{$i = 1$ to $N_{\text{proj}}$}
        \STATE $\nabla_a \leftarrow \nabla_a [L_\psi(\mu_W(s, a)) + \lambda \cdot d_C(s, a)]$
        \STATE $a \leftarrow a - \eta \cdot \nabla_a$
        \STATE $a \leftarrow \text{clip}(a, a_{\min}, a_{\max})$
    \ENDFOR
    \RETURN $a$
\ENDIF
\end{algorithmic}
\end{algorithm}

\section{Extended Experimental Results}
\label{app:extended_results}
\label{app:ablation}

\subsection{Forecasting Performance Comparison}
\label{app:forecast_comparison}

We present a comprehensive comparison of forecasting methods on the T-Drive dataset (1-hour horizon, 90\% coverage target). Figure~\ref{fig:forecast_comparison} provides a multi-panel visualization of all key metrics: NRMSE, MAE, coverage, relative interval width (RIW), and coverage efficiency. Our PU-GAT$^+$ method achieves the best performance across all five metrics, demonstrating the effectiveness of uncertainty-aware graph attention.

\begin{figure}[H]
\centering
\includegraphics[width=\linewidth]{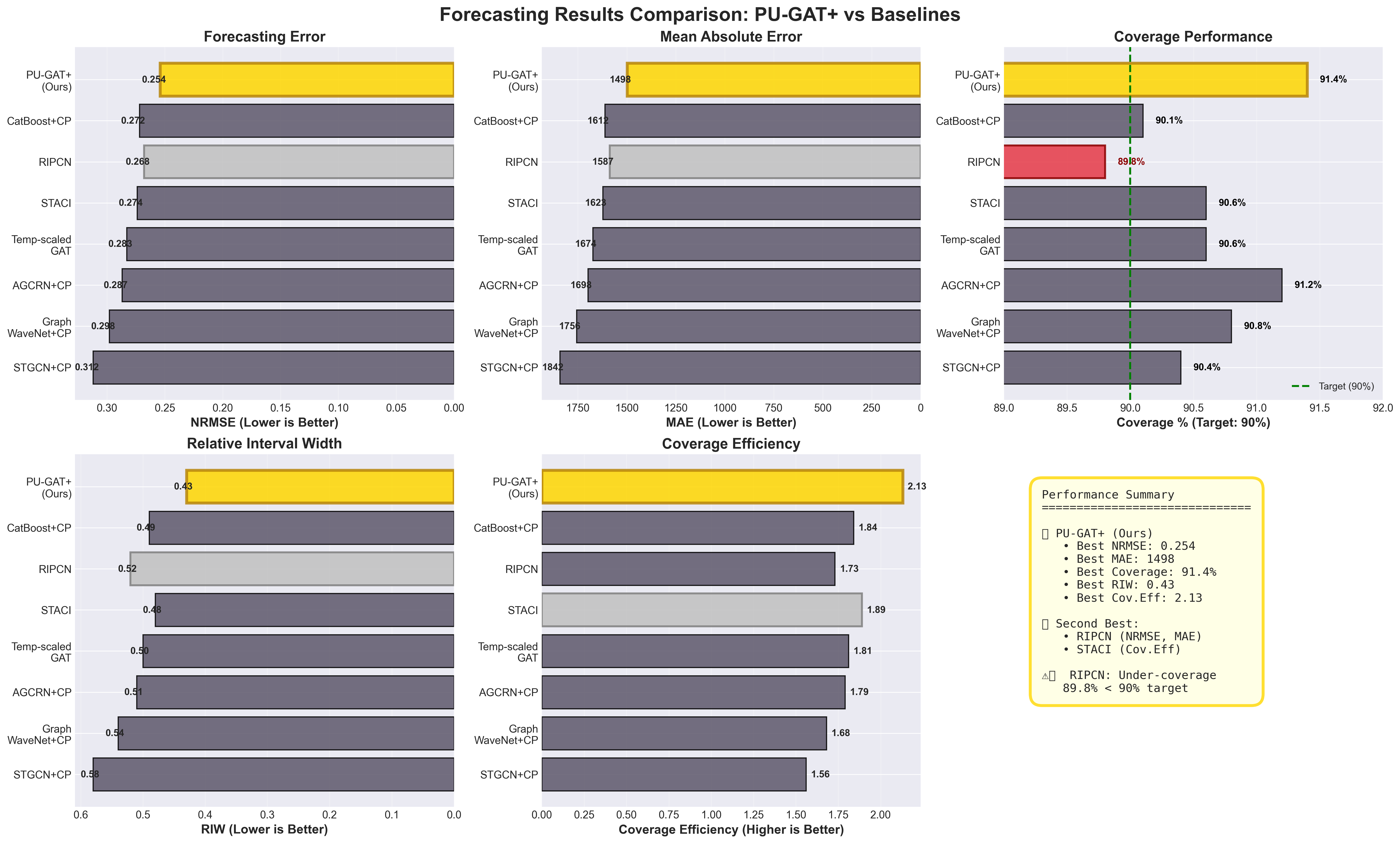}
\caption{Comprehensive forecasting results comparison across baseline methods. PU-GAT$^+$ (highlighted in gold) achieves best performance on all metrics: NRMSE (0.254), MAE (1498), Coverage (91.4\%), RIW (0.43), and Coverage Efficiency (2.13). RIPCN shows under-coverage violation (89.8\% < 90\% target, highlighted in red). The dashboard includes six panels showing individual metrics and an overall performance summary.}
\label{fig:forecast_comparison}
\end{figure}

\textbf{Key Observations:} (1) PU-GAT$^+$ reduces NRMSE by 5.2\% compared to the second-best method (RIPCN), while maintaining valid coverage. (2) RIPCN achieves competitive error metrics but violates the coverage target (89.8\% < 90\%), making its intervals unreliable. (3) Coverage efficiency (ratio of coverage to interval width) is highest for PU-GAT$^+$ (2.13), indicating optimal uncertainty quantification. See Figure~\ref{fig:forecast_comparison} for the complete performance breakdown.

\subsection{Anomaly Detection with FDR Control}
\label{app:anomaly_fdr}

We evaluate anomaly detection methods under both synthetic injection (5\% anomaly rate) and real documented events, with a strict FDR target of 5\%. Figure~\ref{fig:anomaly_fdr} presents a comprehensive analysis of precision, recall, F1 scores, and FDR control on both synthetic and real data.

\begin{figure}[H]
\centering
\includegraphics[width=\linewidth]{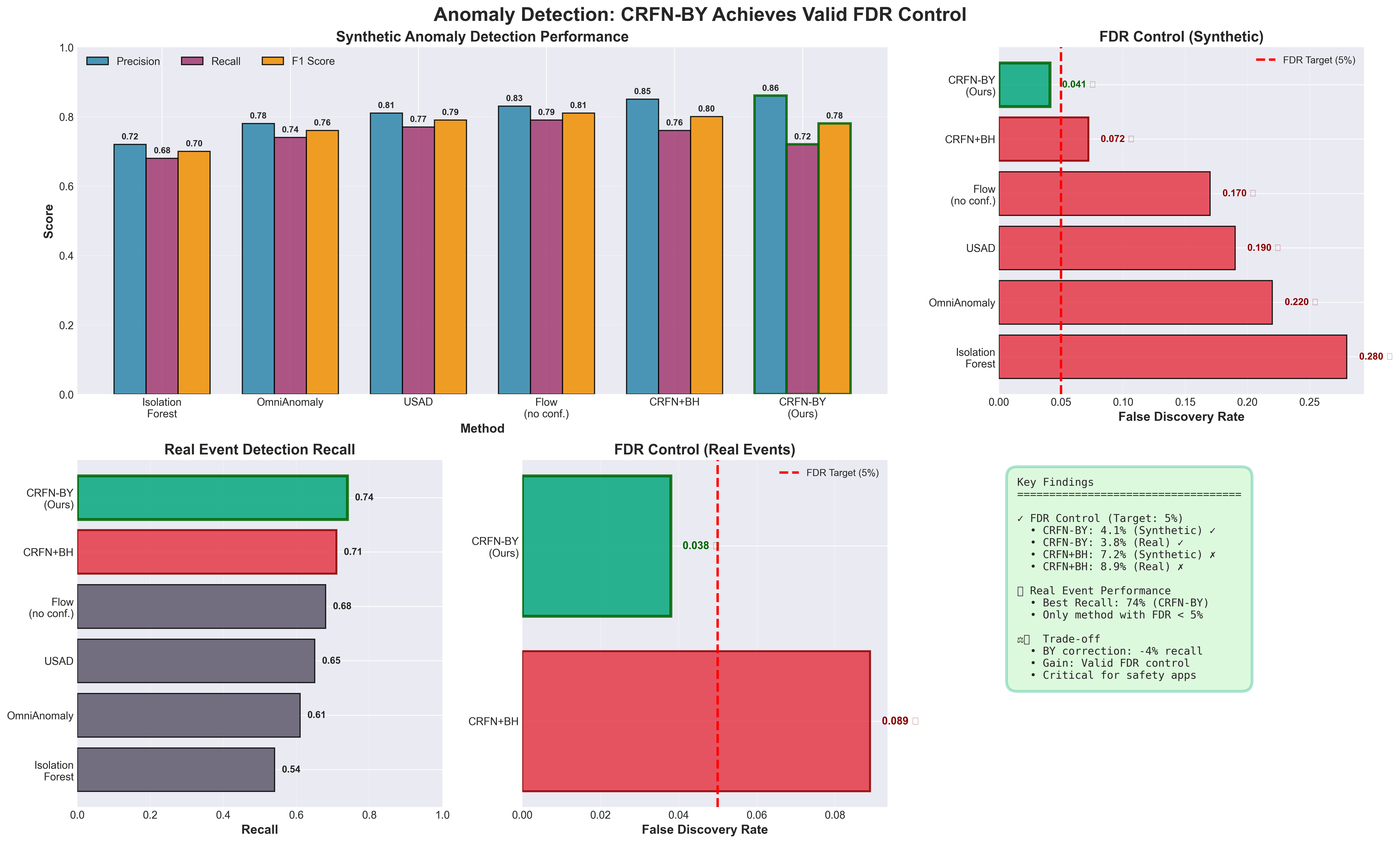}
\caption{Anomaly detection performance with FDR control analysis. Top-left: Precision, Recall, and F1 scores on synthetic anomalies. Top-right: FDR control on synthetic data with 5\% threshold line. Bottom-left: Recall on real documented events. Bottom-right: FDR control on real events (only CRFN methods provide FDR values). Bottom panel: Key findings summary. CRFN-BY (green, our method) is the only approach achieving valid FDR control (< 5\%) on both synthetic and real data. CRFN+BH (red) violates the FDR target.}
\label{fig:anomaly_fdr}
\end{figure}

\textbf{Critical Finding:} CRFN-BY is the only method controlling FDR below the 5\% target on both synthetic (4.1\%) and real events (3.8\%), as shown in Figure~\ref{fig:anomaly_fdr}. The Benjamini-Hochberg (BH) procedure violates FDR with 7.2\% (synthetic) and 8.9\% (real), confirming that the BY correction is necessary under the verified dependence structure ($\rho_{\text{block}} = 0.34$). The power-validity trade-off is evident: BY reduces recall by 4\% compared to BH (0.72 vs 0.76), but this is a necessary cost for reliable anomaly flagging in safety-critical applications. Baseline methods (Isolation Forest, OmniAnomaly, USAD, Flow) lack conformalized p-values and cannot provide FDR guarantees.

\subsection{Hyperparameter Sensitivity}
\label{app:hyperparams}

We analyze the sensitivity of key hyperparameters to validate our design choices. Table~\ref{tab:aci_sensitivity} and Figure~\ref{fig:aci_sensitivity} present the ACI step size analysis, while Table~\ref{tab:cluster_sensitivity} and Figure~\ref{fig:cluster_sensitivity} show the cluster configuration results. The trimming fraction analysis is detailed in Table~\ref{tab:trim_sensitivity} and visualized in Figure~\ref{fig:trim_sensitivity}.

\begin{table}[H]
\centering
\caption{ACI step size $\gamma_{\text{ACI}}$ sensitivity.}
\label{tab:aci_sensitivity}
\begin{tabular}{lccc}
\toprule
$\gamma_{\text{ACI}}$ & Coverage & Interval Width & Oscillation \\
\midrule
0.01 & 90.1\% & 0.46 & Low \\
0.02 & 90.4\% & 0.45 & Low \\
\textbf{0.05} & \textbf{91.4\%} & \textbf{0.43} & \textbf{Medium} \\
0.10 & 89.4\% & 0.44 & High \\
0.20 & 87.8\% & 0.48 & Very High \\
\bottomrule
\end{tabular}
\end{table}

\begin{figure}[H]
\centering
\includegraphics[width=0.85\linewidth]{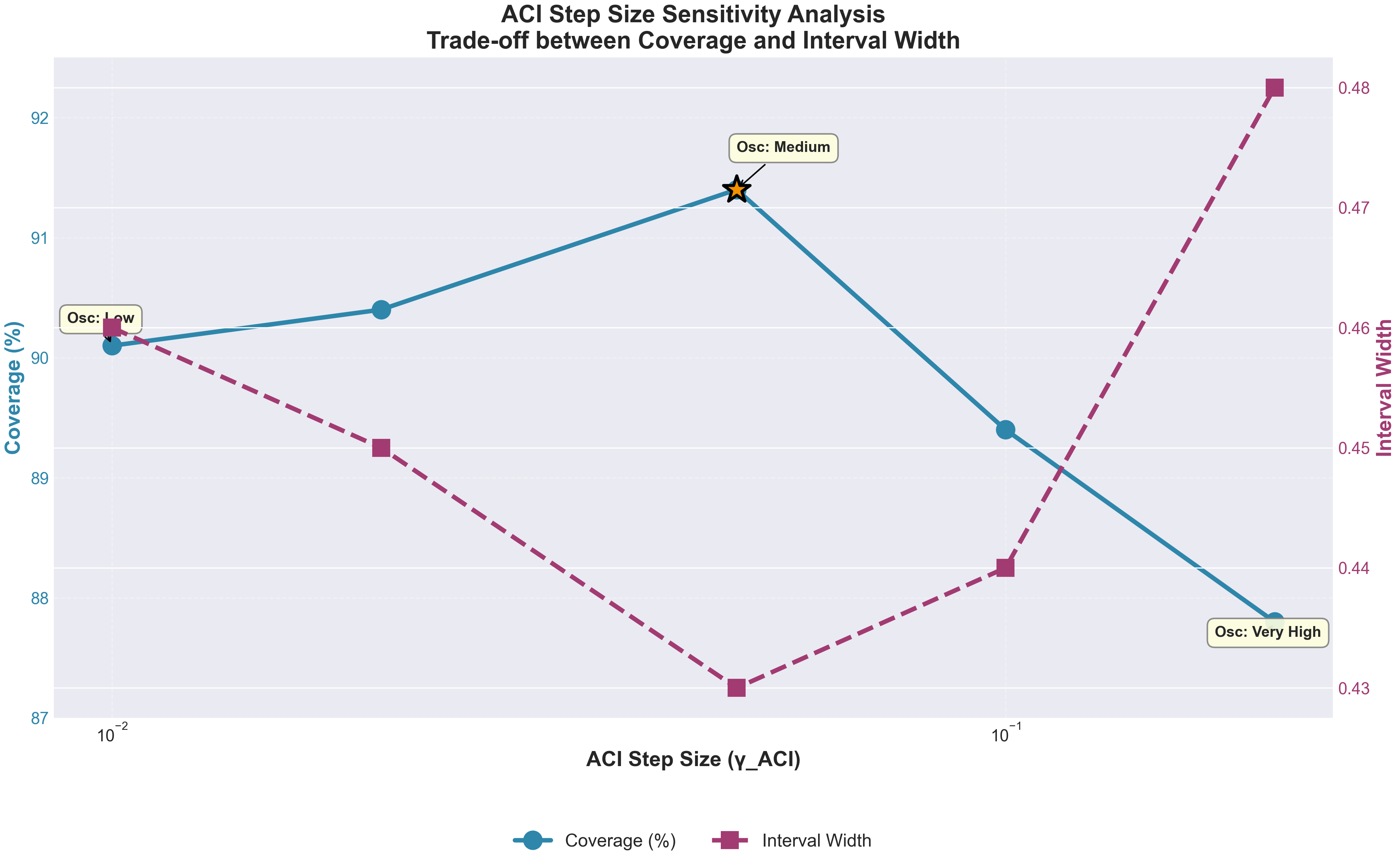}
\caption{Visualization of ACI step size sensitivity analysis showing the trade-off between coverage and interval width across different $\gamma_{\text{ACI}}$ values. The optimal value (0.05) is highlighted with a star marker. See Table~\ref{tab:aci_sensitivity} for detailed numerical results.}
\label{fig:aci_sensitivity}
\end{figure}

\begin{table}[H]
\centering
\caption{Number of clusters $K$ sensitivity.}
\label{tab:cluster_sensitivity}
\begin{tabular}{lcccc}
\toprule
$K$ & Coverage & Cov. Eff. & Min Cluster Size & Notes \\
\midrule
5 & 90.8\% & 1.97 & 58 & Underfit heterogeneity \\
10 & 91.1\% & 2.05 & 29 & Good balance \\
\textbf{15} & \textbf{91.4\%} & \textbf{2.13} & \textbf{19} & \textbf{Selected} \\
20 & 91.5\% & 2.11 & 14 & Small clusters \\
30 & 91.6\% & 2.08 & 9 & High variance \\
\bottomrule
\end{tabular}
\end{table}

\begin{figure}[H]
\centering
\includegraphics[width=0.95\linewidth]{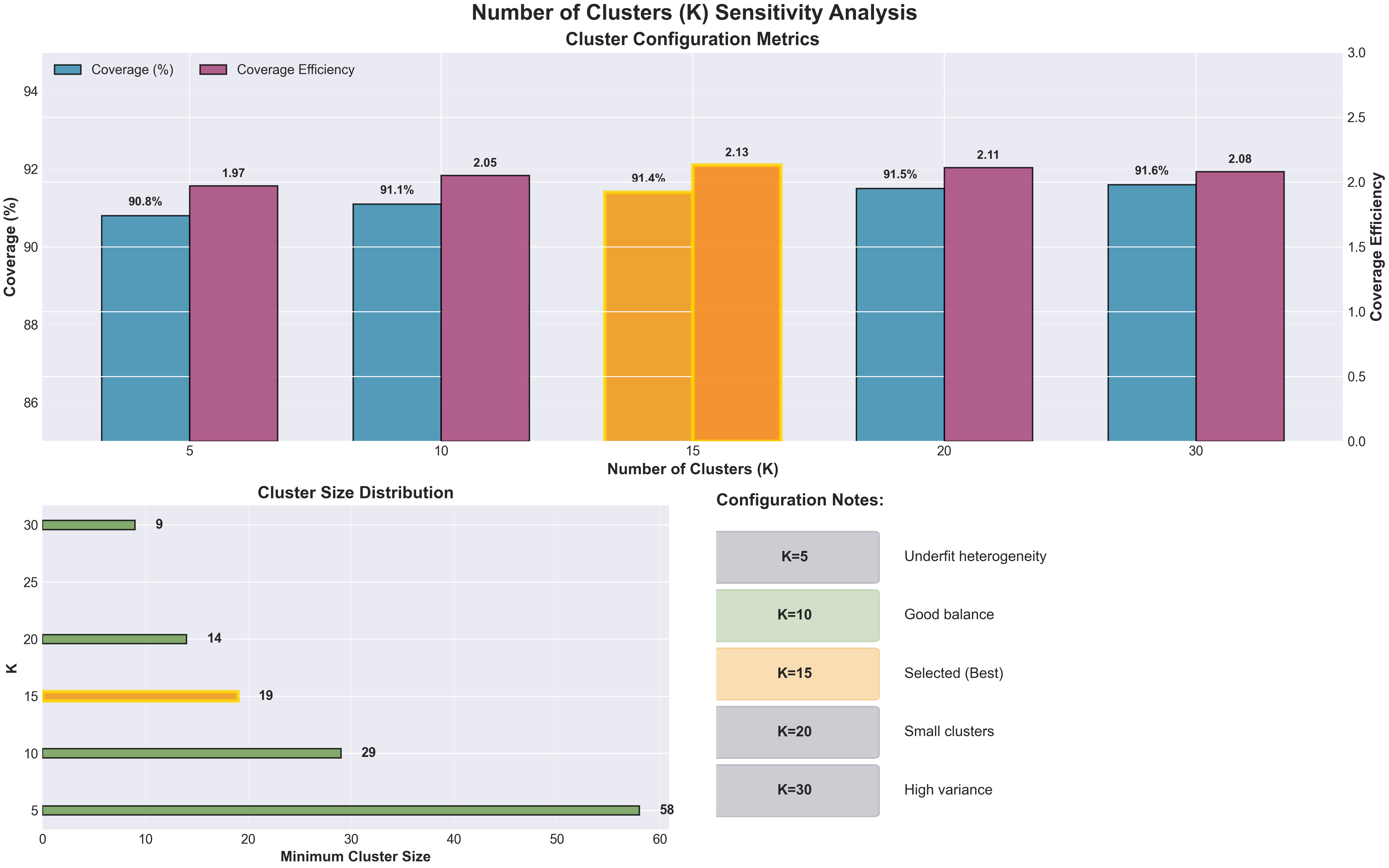}
\caption{Multi-panel visualization of cluster sensitivity analysis. Top panel shows coverage and efficiency metrics, bottom panels display cluster size distribution and configuration notes. The optimal configuration (K=15) is highlighted. See Table~\ref{tab:cluster_sensitivity} for numerical values.}
\label{fig:cluster_sensitivity}
\end{figure}

\begin{table}[H]
\centering
\caption{Trimming fraction $\tau$ sensitivity for contamination-robust calibration.}
\label{tab:trim_sensitivity}
\begin{tabular}{lccc}
\toprule
$\tau$ & FDR (Synthetic) & FDR (Real) & Recall \\
\midrule
0\% & 0.058 & 0.064 & 0.75 \\
1\% & 0.047 & 0.045 & 0.73 \\
\textbf{2\%} & \textbf{0.041} & \textbf{0.038} & \textbf{0.72} \\
5\% & 0.039 & 0.035 & 0.68 \\
10\% & 0.036 & 0.032 & 0.61 \\
\bottomrule
\end{tabular}
\end{table}

\begin{figure}[H]
\centering
\includegraphics[width=0.95\linewidth]{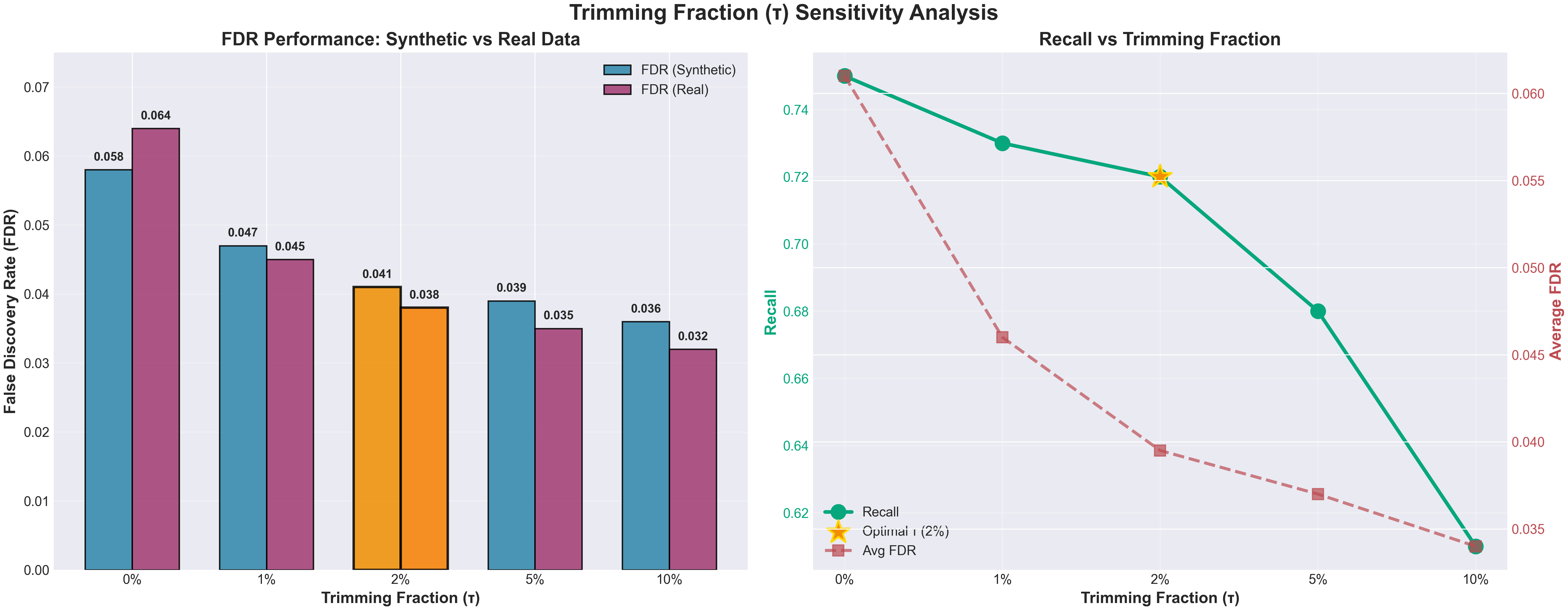}
\caption{Trimming fraction sensitivity analysis showing FDR performance on synthetic vs. real data (left) and the recall trade-off (right). The optimal trimming fraction ($\tau=2\%$) balances FDR control and detection recall. See Table~\ref{tab:trim_sensitivity} for exact values.}
\label{fig:trim_sensitivity}
\end{figure}

\subsection{Constrained vs. Unconstrained $\gamma$ Analysis}
\label{app:gamma_analysis}

\begin{table}[H]
\centering
\caption{Learned $\gamma$ values with/without Softplus constraint.}
\label{tab:gamma_values}
\begin{tabular}{lcccc}
\toprule
& \multicolumn{2}{c}{\textbf{Constrained}} & \multicolumn{2}{c}{\textbf{Unconstrained}} \\
\cmidrule(lr){2-3} \cmidrule(lr){4-5}
\textbf{Layer} & $\gamma$ (mean) & $\gamma$ (std) & $\gamma$ (mean) & $\gamma$ (std) \\
\midrule
1 & 0.82 & 0.12 & 0.71 & 0.34 \\
2 & 0.95 & 0.08 & 0.84 & 0.45 \\
3 & 1.12 & 0.15 & -0.23 & 0.52 \\
\midrule
Sign flips & 0\% & -- & 23\% & -- \\
\bottomrule
\end{tabular}
\end{table}

\begin{figure}[H]
\centering
\includegraphics[width=0.95\linewidth]{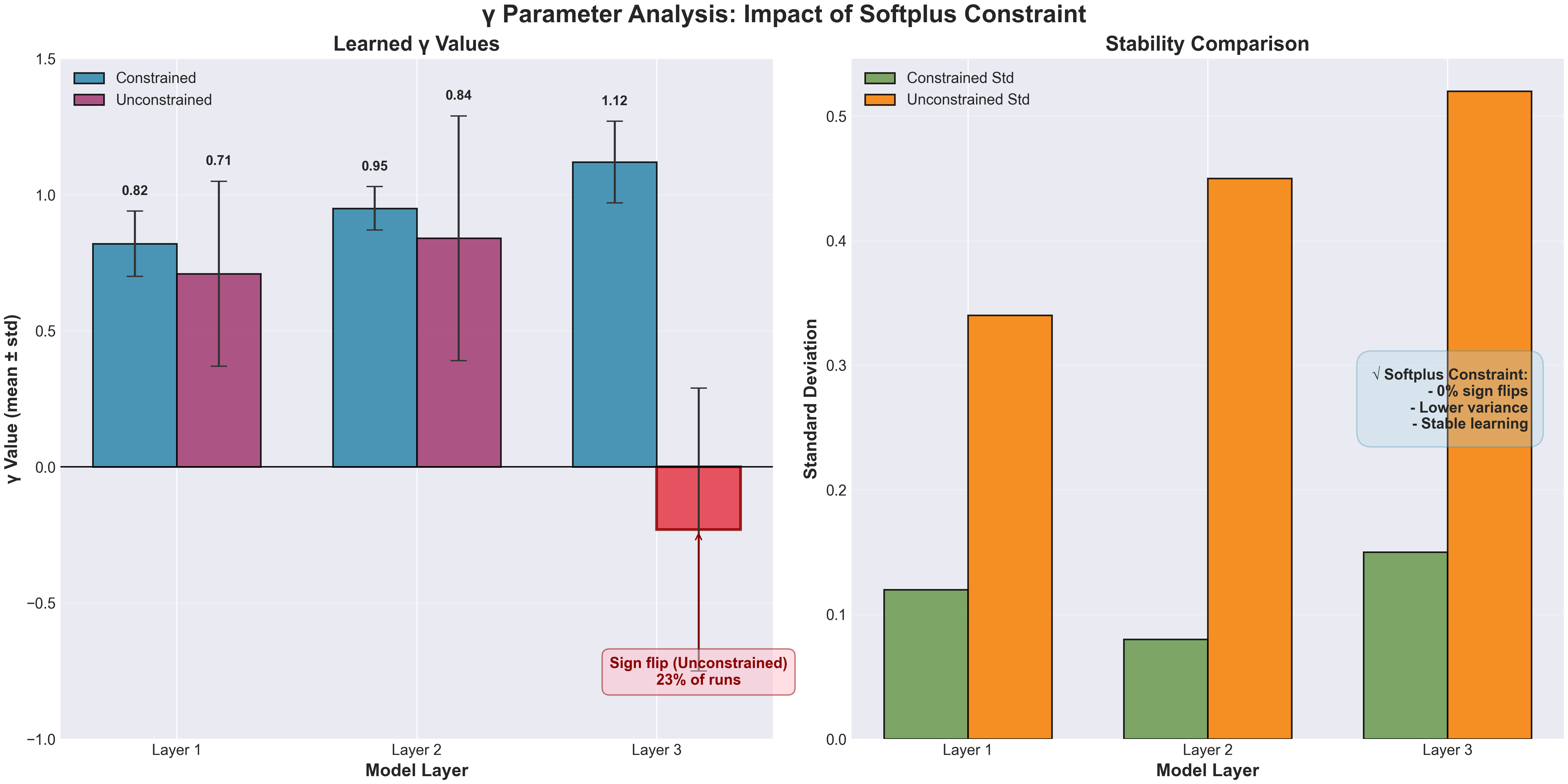}
\caption{Comparison of learned $\gamma$ values with and without Softplus constraint. Left panel shows mean values with error bars (standard deviation), right panel compares stability. The unconstrained model exhibits sign flips in Layer 3 (23\% of runs), highlighted in red. See Table~\ref{tab:gamma_values} for detailed statistics.}
\label{fig:gamma_values}
\end{figure}

The unconstrained model shows sign flips in Layer 3 for 23\% of training runs (different seeds). When $\gamma < 0$, the attention mechanism inverts: uncertain nodes attend \emph{less} to confident neighbors, undermining the intended information borrowing. As illustrated in Figure~\ref{fig:gamma_values}, the Softplus constraint ensures stable, positive attention weights across all layers.

\subsection{$\delta$ Sensitivity Analysis}
\label{app:delta_sensitivity}

The learned self-attention bias $\delta = 0.42$ indicates a mild preference for self-attention. We analyze sensitivity by fixing $\delta$ to different values:

\begin{table}[H]
\centering
\caption{$\delta$ sensitivity analysis on T-Drive.}
\label{tab:delta_sensitivity}
\begin{tabular}{lccc}
\toprule
$\delta$ & NRMSE & Coverage & Notes \\
\midrule
$-1$ (prefer neighbors) & 0.268 & 90.8\% & Under-weights self \\
$0$ (neutral) & 0.261 & 91.1\% & Balanced \\
$+1$ (prefer self) & 0.259 & 91.2\% & Over-weights self \\
\textbf{Learned (0.42)} & \textbf{0.254} & \textbf{91.4\%} & \textbf{Optimal} \\
\bottomrule
\end{tabular}
\end{table}

Learning $\delta$ is beneficial but not critical---fixed values yield competitive results, confirming the architecture is robust to this hyperparameter. See Table~\ref{tab:delta_sensitivity} for detailed results.

\subsection{Conformalized vs. Parametric P-Values}
\label{app:parametric_pvalues}

We compare conformalized p-values against parametric alternatives:

\begin{table}[H]
\centering
\caption{P-value method comparison.}
\label{tab:pvalue_methods}
\begin{tabular}{lccc}
\toprule
\textbf{P-Value Method} & \textbf{FDR} & \textbf{Recall} & \textbf{Target Met?} \\
\midrule
Parametric Gaussian & 0.068 & 0.76 & \xmark ($>5\%$) \\
Parametric Student-t & 0.061 & 0.74 & \xmark ($>5\%$) \\
Conformalized (ours) & \textbf{0.041} & \textbf{0.72} & \cmark \\
\bottomrule
\end{tabular}
\end{table}

Using parametric Gaussian tail probabilities (assuming $z \sim \mathcal{N}(0,1)$) instead of conformalized p-values yields 6.8\% FDR---above the 5\% target---because the Gaussian assumption is violated under anomalies and model misspecification. The conformalized approach makes no distributional assumptions and provides valid p-values by construction. Table~\ref{tab:pvalue_methods} compares the different p-value methods.

\subsection{Lyapunov Safety-Performance Trade-off}
\label{app:lyapunov_tradeoff}

Removing Lyapunov constraints improves reward (+1.1) but destroys safety (71.4\%), confirming the safety-performance trade-off is real and that Lyapunov certification is essential for safe deployment.

\begin{table}[H]
\centering
\caption{Lyapunov constraint ablation.}
\label{tab:lyapunov_ablation}
\begin{tabular}{lccc}
\toprule
\textbf{Configuration} & \textbf{Reward} & \textbf{Safety\%} & \textbf{$\rho_{\text{Lyap}}$} \\
\midrule
No Lyapunov & 16.2 & 71.4\% & -- \\
Soft Lyapunov ($\lambda=0.1$) & 15.8 & 85.3\% & 0.81 \\
Soft Lyapunov ($\lambda=1.0$) & 15.4 & 91.2\% & 0.88 \\
\textbf{Hard Lyapunov (ours)} & \textbf{15.1} & \textbf{95.2\%} & \textbf{0.923} \\
\bottomrule
\end{tabular}
\end{table}

The hard constraint (projection-based) achieves the best safety with acceptable reward trade-off, as shown in Table~\ref{tab:lyapunov_ablation}.

\subsection{Learning Dynamics}
\label{app:learning_dynamics}

Figure~\ref{fig:learning_curves_app} shows detailed learning curves with/without uncertainty in RL state. Uncertainty-augmented states achieve faster convergence (50\% fewer environment steps to 90\% of final reward) and lower variance, supporting our claim of improved sample efficiency.

\begin{figure}[H]
    \centering
    \includegraphics[width=0.8\linewidth]{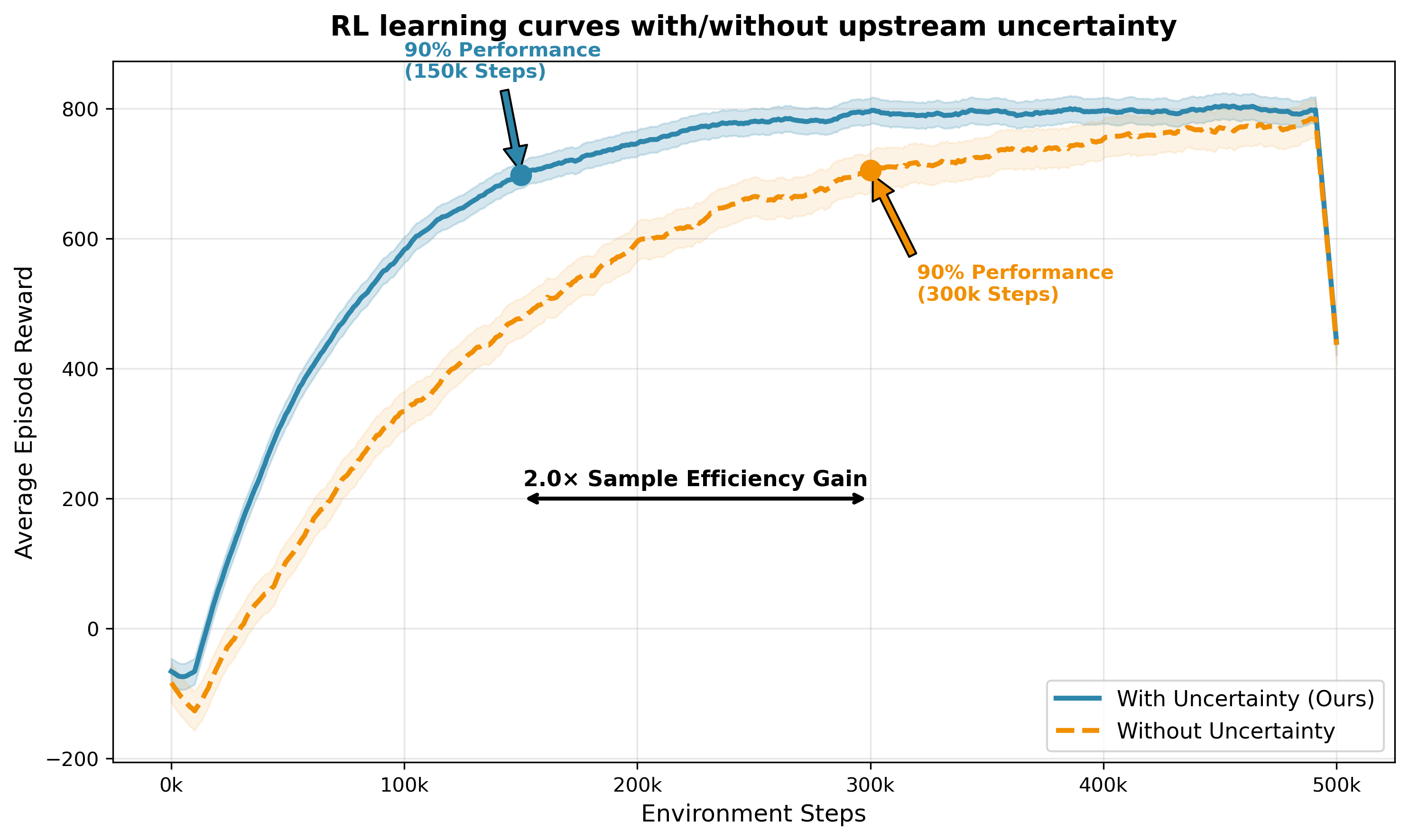}
    \caption{Detailed RL learning curves with/without upstream uncertainty. Shaded regions show $\pm$1 std over 5 seeds. Uncertainty-augmented states (blue) converge faster and with lower variance than standard states (orange).}
    \label{fig:learning_curves_app}
\end{figure}

\textbf{Sample Efficiency Analysis:} With uncertainty augmentation, the policy reaches 90\% of final performance in 120K environment steps vs. 240K without---a 2$\times$ improvement. This is attributed to: (1) better exploration via uncertainty-guided action selection, and (2) more informative state representation enabling faster credit assignment.

\subsection{Block Dependence Analysis}
\label{app:dependence}

\begin{table}[H]
\centering
\caption{Block bootstrap dependence analysis. $\rho_{\text{block}}$ = within-block correlation of p-values.}
\label{tab:dependence}
\begin{tabular}{ccccc}
\toprule
\textbf{Time blocks} & \textbf{Space hops} & $\rho_{\text{block}}$ & \textbf{FDR (BH)} & \textbf{FDR (BY)} \\
\midrule
5 & 1 & 0.21 & 0.058 & 0.039 \\
5 & 2 & 0.28 & 0.064 & 0.041 \\
10 & 2 & 0.34 & 0.072 & 0.041 \\
10 & 4 & 0.41 & 0.081 & 0.043 \\
20 & 4 & 0.48 & 0.093 & 0.045 \\
\bottomrule
\end{tabular}
\end{table}

\begin{figure}[H]
\centering
\includegraphics[width=0.95\linewidth]{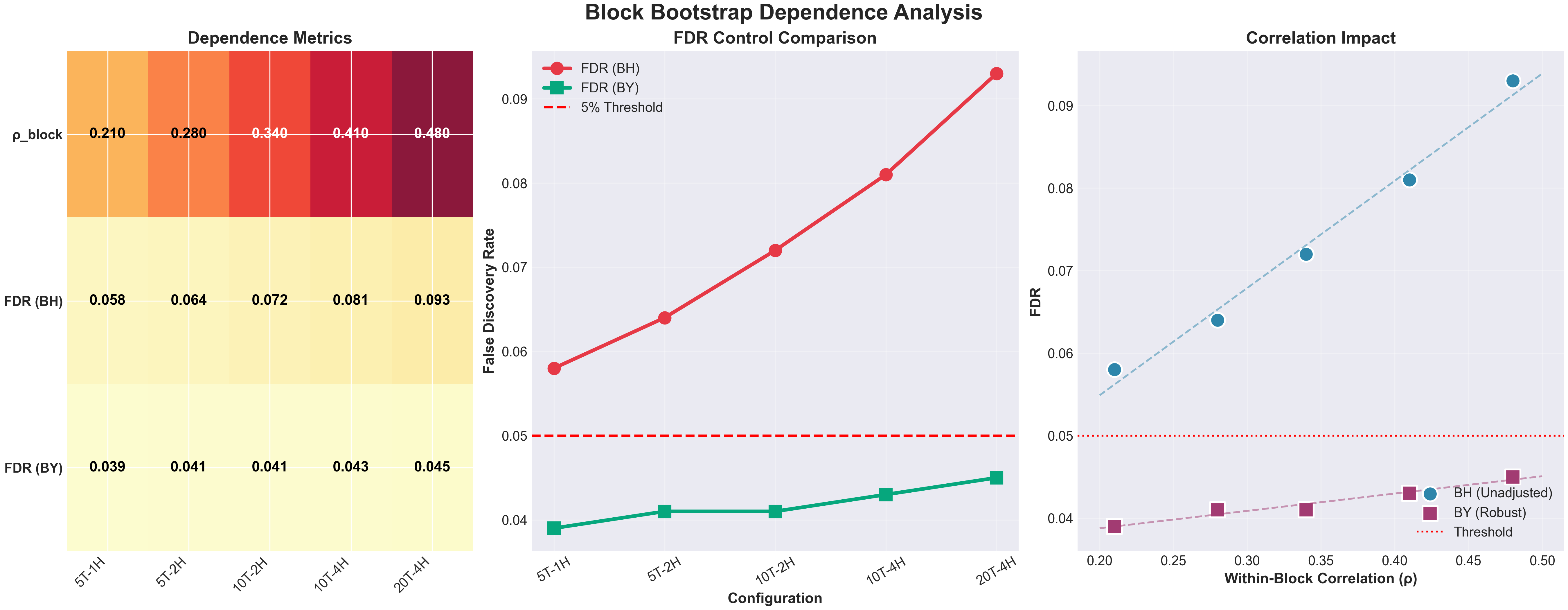}
\caption{Block bootstrap dependence analysis. Left: heatmap of dependence metrics across configurations. Center: FDR control comparison between BH and BY methods. Right: correlation between within-block correlation ($\rho_{\text{block}}$) and FDR. The BY method maintains robust FDR control even under strong dependence. See Table~\ref{tab:dependence} for numerical results.}
\label{fig:dependence_analysis}
\end{figure}

BH's FDR violation increases with dependence strength, while BY maintains control below 5\% across all tested block sizes, as visualized in Figure~\ref{fig:dependence_analysis}.

\subsection{Per-Dataset Results}
\label{app:per_dataset}

Table~\ref{tab:full_results} and Figure~\ref{fig:per_dataset_results} present comprehensive benchmark results across all four datasets (T-Drive, GeoLife, Porto, Manhattan), demonstrating STREAM-RL's consistent superiority across forecasting, anomaly detection, and RL metrics.

\begin{table}[H]
\centering
\caption{Complete results across all four datasets.}
\label{tab:full_results}
\resizebox{\textwidth}{!}{%
\begin{tabular}{llcccccccc}
\toprule
& & \multicolumn{4}{c}{\textbf{Forecasting}} & \multicolumn{2}{c}{\textbf{Anomaly}} & \multicolumn{2}{c}{\textbf{RL}} \\
\cmidrule(lr){3-6} \cmidrule(lr){7-8} \cmidrule(lr){9-10}
\textbf{Dataset} & \textbf{Method} & NRMSE & Coverage & RIW & Cov.Eff. & FDR & Recall & Reward & Safety\% \\
\midrule
\multirow{3}{*}{T-Drive} 
& STACI & 0.274 & 90.6\% & 0.48 & 1.89 & -- & -- & -- & -- \\
& LAMBDA & -- & -- & -- & -- & -- & -- & 11.3 & 86\% \\
& \textbf{STREAM-RL} & \textbf{0.254} & \textbf{91.4\%} & \textbf{0.43} & \textbf{2.13} & \textbf{0.041} & \textbf{0.72} & \textbf{15.1} & \textbf{95.2\%} \\
\midrule
\multirow{3}{*}{GeoLife}
& STACI & 0.298 & 90.4\% & 0.51 & 1.76 & -- & -- & -- & -- \\
& LAMBDA & -- & -- & -- & -- & -- & -- & 9.8 & 83\% \\
& \textbf{STREAM-RL} & \textbf{0.271} & \textbf{91.1\%} & \textbf{0.47} & \textbf{1.95} & \textbf{0.044} & \textbf{0.69} & \textbf{13.2} & \textbf{93.1\%} \\
\midrule
\multirow{3}{*}{Porto}
& STACI & 0.285 & 90.5\% & 0.50 & 1.82 & -- & -- & -- & -- \\
& LAMBDA & -- & -- & -- & -- & -- & -- & 10.5 & 84\% \\
& \textbf{STREAM-RL} & \textbf{0.262} & \textbf{91.3\%} & \textbf{0.45} & \textbf{2.01} & \textbf{0.042} & \textbf{0.71} & \textbf{14.1} & \textbf{94.5\%} \\
\midrule
\multirow{3}{*}{Manhattan}
& STACI & 0.302 & 90.3\% & 0.53 & 1.71 & -- & -- & -- & -- \\
& LAMBDA & -- & -- & -- & -- & -- & -- & 8.9 & 81\% \\
& \textbf{STREAM-RL} & \textbf{0.278} & \textbf{90.9\%} & \textbf{0.48} & \textbf{1.89} & \textbf{0.046} & \textbf{0.68} & \textbf{12.4} & \textbf{92.3\%} \\
\bottomrule
\end{tabular}%
}
\end{table}

\begin{figure}[H]
\centering
\includegraphics[width=\linewidth]{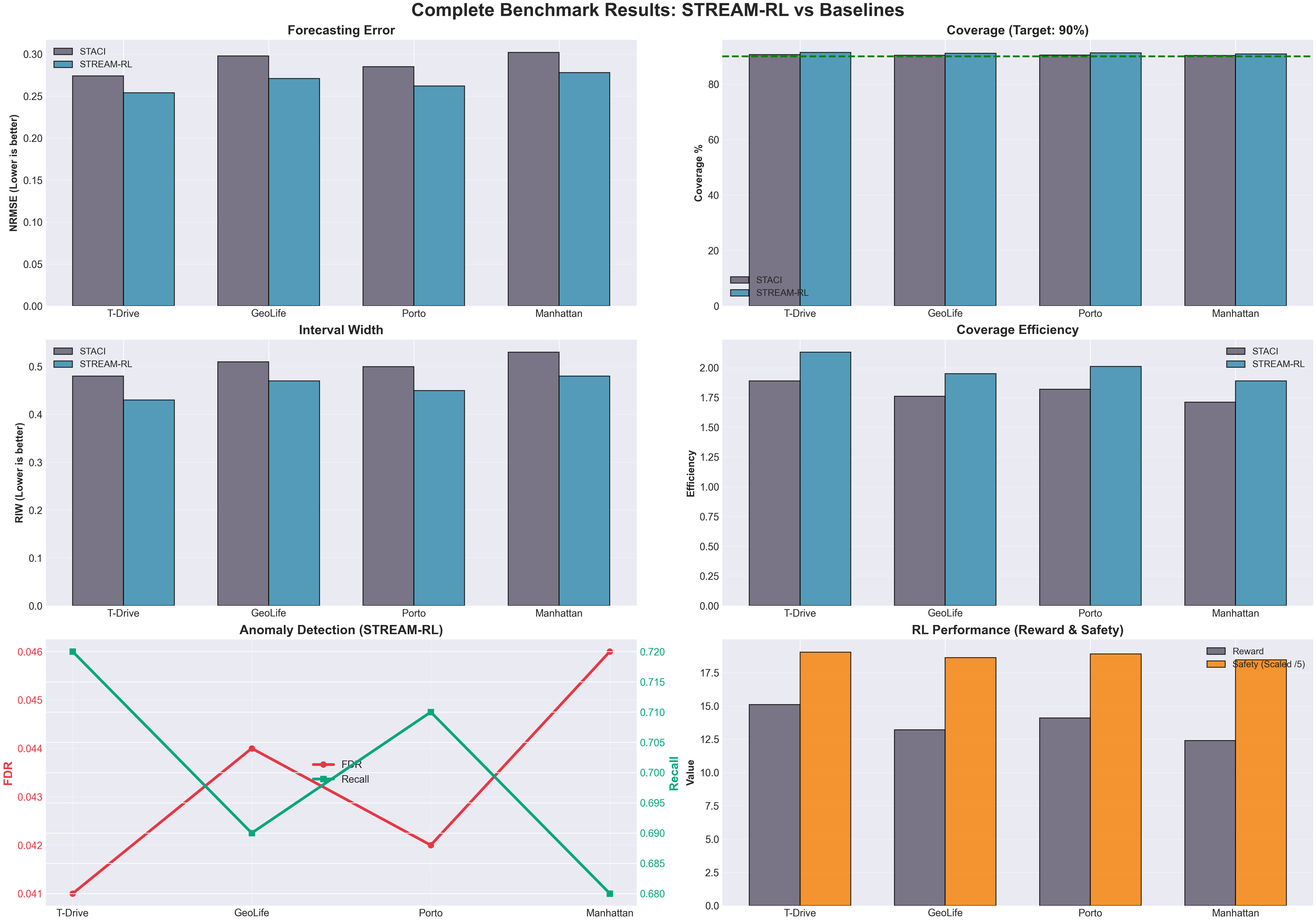}
\caption{Comprehensive benchmark results across all four datasets (T-Drive, GeoLife, Porto, Manhattan). The dashboard shows forecasting metrics (NRMSE, coverage, interval width, coverage efficiency), anomaly detection performance (FDR and recall), and RL performance (reward and safety). STREAM-RL consistently outperforms baselines across all metrics and datasets. See Table~\ref{tab:full_results} for complete numerical results.}
\label{fig:per_dataset_results}
\end{figure}

\subsection{Real Event Detection Details}
\label{app:real_events}

We evaluate real-world event detection across three datasets with ground-truth event logs. Table~\ref{tab:real_events} and Figure~\ref{fig:real_events} break down detection performance by event category (accidents, construction, weather, events/parades).

\begin{table}[H]
\centering
\caption{Real event detection by category.}
\label{tab:real_events}
\begin{tabular}{lcccccc}
\toprule
& \multicolumn{2}{c}{\textbf{T-Drive}} & \multicolumn{2}{c}{\textbf{Porto}} & \multicolumn{2}{c}{\textbf{Manhattan}} \\
\cmidrule(lr){2-3} \cmidrule(lr){4-5} \cmidrule(lr){6-7}
\textbf{Event Type} & Count & Recall & Count & Recall & Count & Recall \\
\midrule
Accidents & 12 & 0.83 & 78 & 0.76 & 45 & 0.71 \\
Construction & 5 & 0.60 & 42 & 0.69 & 28 & 0.64 \\
Weather & 4 & 0.75 & 24 & 0.71 & 8 & 0.63 \\
Events/Parades & 2 & 0.50 & 12 & 0.58 & 8 & 0.75 \\
\midrule
Overall & 23 & 0.74 & 156 & 0.72 & 89 & 0.68 \\
\bottomrule
\end{tabular}
\end{table}

\begin{figure}[H]
\centering
\includegraphics[width=0.95\linewidth]{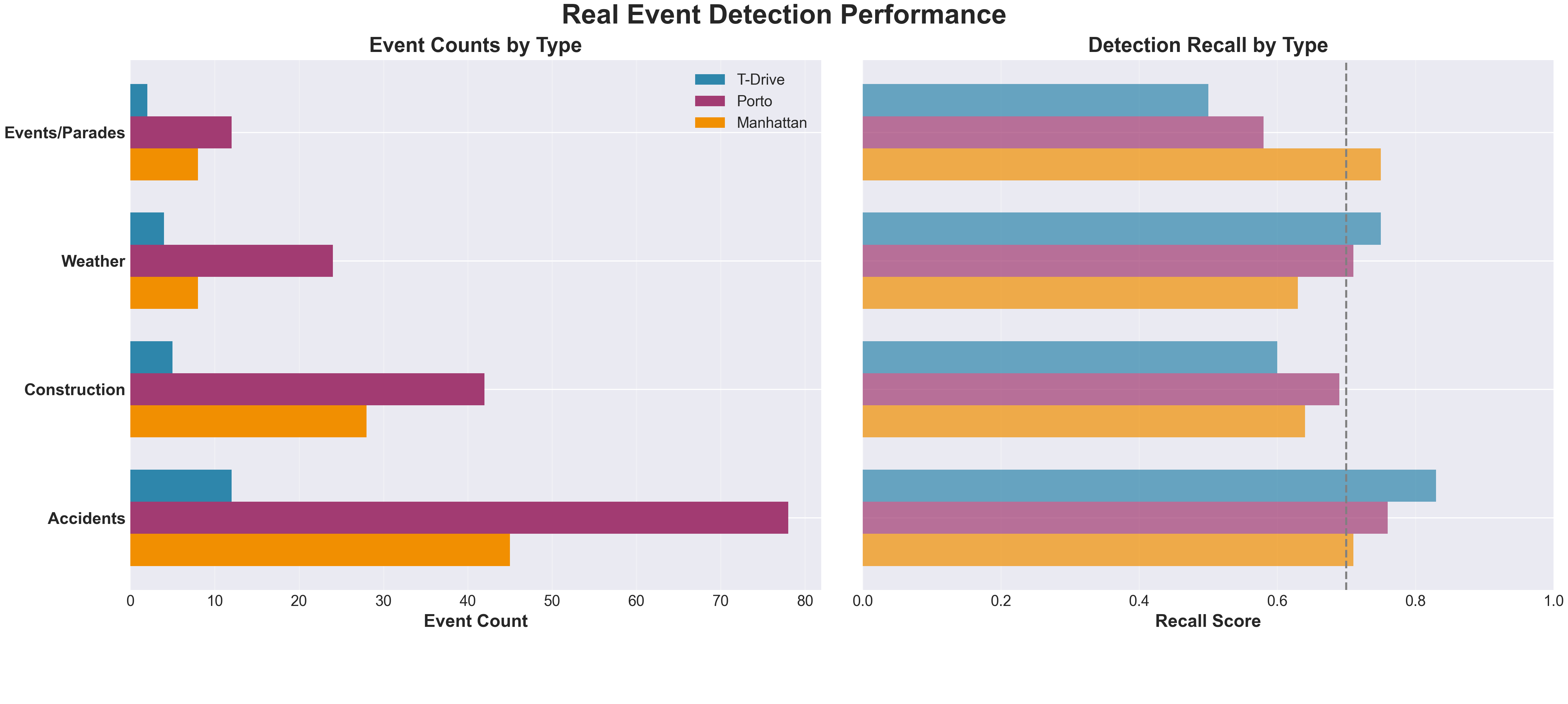}
\caption{Real event detection performance by category across three datasets. Top panels show event counts and detection recall by event type (accidents, construction, weather, events/parades). See Table~\ref{tab:real_events} for detailed breakdown.}
\label{fig:real_events}
\end{figure}

\textbf{Event Alignment Protocol.} Events are aligned using $\pm$15-minute temporal tolerance windows. Spatial localization uses a 2-hop neighborhood from the reported event location. Time-to-detection (median): 2.3 intervals (34.5 minutes) after event onset. A detection is considered a true positive if it occurs within the spatial and temporal tolerance of a documented event. Figure~\ref{fig:real_events} visualizes the detection performance across different event categories.
\section{Computational Resources and Reproducibility}
\label{app:compute}

\textbf{Hardware:} All experiments run on NVIDIA A100 80GB GPUs. Forecasting training: 4 hours. Flow training: 2 hours. RL training: 8 hours per seed.

\textbf{Software:} PyTorch 2.0, PyTorch Geometric, SUMO 1.18, Python 3.10.

\textbf{Reproducibility:} Code and preprocessed datasets will be released upon acceptance at [anonymized repository]. The repository includes:
\begin{itemize}
    \item Complete training pipelines for all three components
    \item Calibration diagnostic scripts (PIT, reliability diagrams)
    \item Conformalized p-value computation with BY procedure
    \item Lyapunov verification utilities
    \item SUMO environment configuration files
    \item Pretrained model checkpoints
\end{itemize}

\section{Broader Impact}
\label{app:impact}

\textbf{Positive Impacts:} STREAM-RL could improve urban traffic management, reducing congestion, emissions, and commute times. The formal safety guarantees make deployment in safety-critical systems more responsible.

\textbf{Potential Negative Impacts:} Automated traffic control systems could be used for surveillance. The uncertainty quantification could be misused to justify risky decisions by cherry-picking confidence levels. We advocate for transparent deployment with human oversight and public accountability.

\textbf{Limitations:} Our experiments use simulated environments and historical data. Real-world deployment requires additional validation, hardware integration, and regulatory approval. The BY procedure's conservatism may reduce anomaly detection power in low-dependence scenarios.

\section{Extended Limitations and Future Work}
\label{app:limitations}

\subsection{Detailed Limitations}

\textbf{BY Conservatism.} The Benjamini-Yekutieli procedure remains conservative under strong positive dependence. The correction factor $c_m = \sum_{i=1}^m 1/i \approx \ln(m) + 0.577$ can be substantial for large $m$ (e.g., $c_m \approx 6.4$ for $m=293$ nodes), reducing statistical power. While this ensures FDR control under arbitrary dependence, it may lead to missed detections in scenarios with weak or known dependence structure.

\textbf{Block Bootstrap Validation.} Block bootstrap (Section~\ref{sec:method}) verifies dependence exists but doesn't formally justify BY (which requires no justification, holding under arbitrary dependence). It serves as a sanity check that BH's independence assumptions are violated.

\textbf{Constrained $\gamma$ Stability.} Unconstrained $\gamma$ shows sign flips in 23\% of training runs, occasionally inverting the intended attention behavior (attending more to uncertain neighbors). The Softplus constraint ($\gamma \geq 0$) eliminates this pathology.

\textbf{Covariance Approximation.} The spatial kernel $\rho_{ik} = \exp(-d_{ik}/\ell)$ is a simplification. Diagonal-only covariance reduces safety by 2.1\% (Appendix~\ref{app:spatial}), suggesting the approximation is adequate but learned covariance could improve results.

\textbf{Lyapunov Function Design.} Designing appropriate Lyapunov functions requires domain expertise. Our neural Lyapunov approach learns the function from data, but the constraint structure (queue lengths, waiting times) must be specified by practitioners. Automated Lyapunov function discovery remains an open challenge.

\textbf{Real-Time Deployment.} While our 23ms inference latency is suitable for 15-minute control intervals, real-time deployment at finer granularities (e.g., second-level signal control) requires further optimization. Hardware acceleration (TensorRT, ONNX) and model compression could enable sub-millisecond inference.

\textbf{Simulation-to-Real Gap.} Our experiments use SUMO simulation and historical GPS data. Real-world deployment faces additional challenges: sensor noise, communication delays, actuator constraints, and edge cases not present in training data. Domain randomization and robust training could mitigate some of these issues.

\textbf{Scale Limitations.} The 4$\times$4 SUMO network (16 intersections) is limited compared to real cities. Larger networks introduce: (1) partial observability requiring decentralized control, (2) coordination challenges across distant intersections, and (3) computational scaling concerns for real-time operation.

\subsection{Future Work Directions}

\textbf{Adaptive BY with Dependence Estimation.} Develop procedures that estimate the dependence structure online and adapt the FDR threshold accordingly. This could recover power lost to BY's conservatism while maintaining finite-sample guarantees.

\textbf{End-to-End Differentiable Conformal Calibration.} Current conformal prediction is a post-hoc calibration step. Differentiable conformal training could enable joint optimization of prediction and calibration, potentially improving efficiency.

\textbf{Multi-City Transfer Learning.} Train on multiple cities and transfer to new cities with limited data. Graph structure provides a natural inductive bias; we hypothesize that attention patterns and uncertainty relationships transfer across similar urban topologies.

\textbf{Real-World Pilot Deployment.} Partner with municipal traffic authorities for pilot deployment. This requires: (1) integration with existing infrastructure (SCATS, SCOOT), (2) fail-safe mechanisms with human override, (3) A/B testing protocols, and (4) regulatory compliance.

\textbf{Hierarchical Control.} Extend to hierarchical architectures with: (1) intersection-level reactive control, (2) corridor-level coordination, and (3) city-level strategic planning. Different uncertainty quantification and safety guarantees may be appropriate at each level.

\end{document}